%% file: paper.tex
\documentclass[twoside]{article}

\usepackage[accepted]{aistats2020}
\input{preamble}

\hypersetup{
  pdfinfo={
    Title={A Primal-Dual Solver for Large-Scale Tracking-by-Assignment},
    Author={Stefan Haller, Mangal Prakash, Lisa Hutschenreiter, Tobias Pietzsch, Carsten Rother, Florian Jug, Paul Swoboda, Bogdan Savchynskyy},
    Subject={A fast approximate solver for the combinatorial problem known as tracking-by-assignment which is applied to cell tracking.},
    Keywords={optimization, primal-dual, tracking-by-assignment, cell tracking}
  }
}

\begin{document}

\abovedisplayskip .4em plus 1pt
\abovedisplayshortskip 0pt
\belowdisplayskip .4em plus 1pt
\belowdisplayshortskip \belowdisplayskip

\twocolumn[
  \aistatstitle{A Primal-Dual Solver for Large-Scale Tracking-by-Assignment}

  \runningauthor{
    S. Haller,
    M. Prakash,
    L. Hutschenreiter,
    T. Pietzsch,
    C. Rother,
    F. Jug,
    P. Swoboda,
    B. Savchynskyy}

  \aistatsauthor{%
    \bfseries
    Stefan Haller\footnotemark[1],\quad
    Mangal Prakash\footnotemark[2],\quad
    Lisa Hutschenreiter\footnotemark[1],\quad
    Tobias Pietzsch\footnotemark[2],\\
    \bfseries
    Carsten Rother\footnotemark[1],\quad
    Florian Jug\footnotemark[2],\quad
    Paul Swoboda\footnotemark[3],\quad
    Bogdan Savchynskyy\footnotemark[1]}

  \aistatsaddress{%
    \footnotemark[1]Visual Learning Lab, Heidelberg University,\qquad
    \footnotemark[2]Center for Systems Biology, Dresden,\\
    \footnotemark[3]Max-Planck-Institute for Informatics, Saarbrücken}
]

\begin{abstract}
  \vspace*{-3pt}%
  We propose a fast approximate solver for the combinatorial problem known as \emph{tracking-by-assignment}, which we apply to cell tracking.
  The latter plays a key role in discovery in many life sciences, especially in cell and developmental biology.
  So far, in the most general setting this problem was addressed by off-the-shelf solvers like Gurobi, whose run time and memory requirements rapidly grow with the size of the input.
  In contrast, for our method this growth is nearly linear.
  \newline\hspace*{1.5em}%
  Our contribution consists of a new (1)~decomposable compact representation of the problem; (2)~dual block-coordinate ascent method for optimizing the decomposition-based dual; and (3)~primal heuristics that reconstructs a feasible integer solution based on the dual information.
  Compared to solving the problem with Gurobi, we observe an up to~60~times speed-up, while reducing the memory footprint significantly.
  We demonstrate the efficacy of our method on real-world tracking problems.
\end{abstract}

\vspace*{-0.9em}%
\section{Introduction}
\label{sec:introduction}

The tracking problem consists of segmenting images obtained over~$T$ time steps and matching the segments in consecutive images to each other.
In the case of cell tracking each cell in the image~$t$ must be matched to the corresponding cell (or a pair of cells in case of a division) in the image~$t+1$.
Tracking problems are important not only for bioimaging~\cite{Chenouard2014ptc,Ulman2017ctc,Kamentsky2011htt} but also for general computer vision~\cite{wu2015object,kristan2017visual}.
Cell tracking represents one of the hardest types of this problem, due to the existence of cell divisions and the indistinguishability of individual cells.

The most successfully deployed tracking models are typically formulated as integer linear programs (ILPs) and are in general NP-hard.
This makes the commonly used ILP-based optimizations only amenable to moderately sized tracking instances.
With the advent of modern microscopy techniques, this bottleneck became a limiting factor for many real-world applications, leading either to faulty tracking results, or impractical optimization problems.
In order to address run time and scalability issues, we propose a new method to solve tracking-by-assignment problems, whose iterations as well as memory footprint scale nearly linearly with the problem size.
Its convergence speed enables us to obtain high-quality approximate solutions in a fraction of the time required by the best off-the-shelf solvers.

While this paper focuses on biologically motivated object tracking, similar problems also arise in other vision domains~\cite{Wang2014peopleincars,Tang2015multitarget,Iqbal_2017_CVPR}.
We believe that with appropriate modifications our ideas can be applied there as well.

\Paragraph{Related work.}
Any visual tracking contains two key interrelated operations: segmentation and matching.
On one side, matching requires segmentation, on the other side, the segmentation quality can often be significantly improved by the matching results.
There are approaches addressing these two problems jointly~\cite{jug2016moral}, however, the resulting algorithms are quite time consuming.
Less expensive modeling techniques can be categorized into tracking-by-model-evolution and tracking-by-assignment~\cite{jug2014bioimage}.

In \emph{tracking-by-model-evolution}, objects are detected in the first frame and a model of their properties, \eg shape and position, is obtained.
This model is then updated greedily for pairs of neighboring frames, thereby tracking all detected objects.
For sensible results these methods typically require a high temporal resolution~\cite{kausler2012discrete,jug2014bioimage}.
Recently, neural networks for tracking-by-model-evolution have been proposed that jointly tackle segmentation and tracking and have the ability to handle object divisions~\cite{payer2018instance,arbelle2018microscopy}.
They incorporate temporal information, \eg by using LSTMs~\cite{arbelle2018microscopy}.
However, such network based approaches require vast amounts of annotated training data, which is typically not available for biomedical tracking problems.

By contrast, \emph{tracking-by-assignment} first segments potential object candidates in all time frames~\cite{kausler2012discrete,schiegg2013conservation,jug2014optimal,schiegg2014graphical,magnusson2015global,kaiser2018moma}.
There are two important cases: a single segmentation hypothesis per object~\cite{kausler2012discrete,schiegg2013conservation} and multiple ones~\cite{schiegg2014graphical,kaiser2018moma}.
Multiple segmentation hypotheses may correspond to different~(typically overlapping) positions of the same object, to parts of a single object looking like separate ones, or to several objects close to each other looking like a single one.
Apart from better tracking quality, multiple segmentation approaches facilitate user-driven proofreading of automated results~\cite{jug2014tracking,kaiser2018moma}.

Existing optimization methods for tracking-by-assignment fall into two categories:
(\emph{i})~local approaches that attempt to overcome scalability issues by decomposing the overall tracking problem into smaller sub-problems~\cite{Xing_2009,Castanon_2011,jaqaman2008robust}, and
(\emph{ii})~global approaches, treating the whole spatiotemporal problem jointly~\cite{padfield2009coupled,haubold2016generalized,magnusson2015global,Butt_2013}.
While the first type of method scales better with the problem size, the second one leads to better solutions.
Among approaches of the second type we distinguish the work~\cite{padfield2009coupled}, that couples multiple min-cost-flow networks to handle divisions and finds an approximate solution to the overall problem with an off-the-shelf LP solver.
Another notable contribution is the primal heuristics in~\cite{haubold2016generalized}, based on sequentially computing shortest paths in an augmented flow graph which accounts for object divisions.
This work generalizes the approach of~\cite{magnusson2015global} which utilizes the Viterbi algorithm.
However, these works do not handle overlapping segmentation hypotheses.
Additionally, \cite{haubold2016generalized} and \cite{magnusson2015global} do not provide any bounds on the quality of the proposed solutions.
The work~\cite{Butt_2013} employs stochastic gradient descent to maximize a Lagrange dual based on multiple min-cost-flow subproblems, and obtains primal solutions by rounding.
However, the method does not allow for cell divisions.

After all, the most general models including both multiple hypothesis and object divisions have been addressed with off-the-shelf ILP solvers only~\cite{jug2014optimal,jug2014tracking,kaiser2018moma,schiegg2014graphical}.
Building a scalable solver able to compete with, for example, Gurobi in this case seems to be a non-trivial task, which has not been addressed in the literature yet.

\Paragraph{Contributions.}
We propose a new approximate optimization method for cell-tracking problems, which favorably compares with Gurobi in terms of run time and memory footprint, while delivering solutions of a comparable quality.
Our method is able to handle cell divisions and multiple segmentation hypotheses.
Together with an approximate primal solution it provides a lower bound on the optimum.
This is achieved by optimizing the Lagrange dual problem constructed from a new compact decomposition.
To optimize the dual we propose a specialized fast converging algorithm based on the block-coordinate ascent principle.
The approximate primal solution is obtained with a novel primal heuristics based on conflict resolution and greedy elongation of trajectories.
While the dual solver simplifies the objective function by reweighting its costs, the primal one reconstructs an integer solution based on these costs.
We empirically show advantages of our framework on publicly available instances of the cell-tracking challenge~\cite{Ulman2017ctc}, on instances of developing flywing tissue, and on an instance of nuclei tracking in developing drosophila embryos.
These datasets represent different biological applications and exhibit diverse characteristics.
Therefore, we believe our method to be applicable to a wide spectrum of cell-tracking problems.

The focus of our work is to improve the optimization stage of a typical tracking-by-assignment pipeline, therefore, we do not address modeling and segmentation aspects of the problem here.
We assume that for each time step $t\in\{1,\ldots,T\}$ a set of \emph{segmentation hypotheses} is available along with a set of possible transitions and the corresponding costs.

Mathematical proofs and information about our code and models can be found in the supplement.

\section{Standard tracking as ILP}
\label{sec:standard}

\input{figures/decomposition}

The standard modeling approach for tracking-by-assignment is based on its \emph{problem (hyper-)graph} representation~\cite{lou2011deltr,kausler2012discrete,schiegg2013conservation,schiegg2014graphical,jug2014optimal,jug2014tracking}, see Figure~\ref{fig:decomposition} (b).
In the following we omit the prefix \emph{hyper-} and use the \emph{hat} superscript (as in $\hat\SV$ or $\hat\SE$) for the standard problem graph to distinguish it from the graph we propose later.
Nodes $\hat\SV$ of a problem graph $\hat\SG=(\hat\SV,\hat\SE)$ are associated with \emph{finite-valued variables}, and edges $\hat\SE\subseteq 2^{\smash{\hat\SV}}$ correspond to the \emph{coupling constraints} between the respective nodes.
Here, $2^{\smash{\hat\SV}}$ denotes the power set of $\hat\SV$.

\Paragraph{Node set.}
For tracking-by-assignment the set of nodes $\hat\SV$ is divided into disjoint subsets $\hat\SV^t$ corresponding to each time step $t\in\{1,\dots,T\}$, \ie $\hat\SV=\bigcup_{t=1}^{T}\hat\SV^t$.
In its turn, each subset $\hat\SV^t$ is subdivided into a set $\hat\SV^t_\Sdet$ representing the \emph{segmentation hypothesis (detections)} at time step $t$, and a set $\hat\SV^t_\Strans$ representing the possible \emph{transitions (moves, divisions)} from time step $t$ to $t+1$.
We will write $\hat\SV_\Sdet = \bigcup_{t=1}^{T}\hat\SV^t_\Sdet$ and $\hat\SV_\Strans = \bigcup_{t=1}^{T}\hat\SV^t_\Strans$ for the sets of all detection and transition nodes.

Each segmentation hypothesis as well as each transition is associated with a \emph{binary variable}, \ie its value is in the set $\{0,1\}$.
We will refer to the variable corresponding to node $v\in\hat\SV$ as $x_v$, where $x_v\in\{0,1\}$.
A variable is said to be \emph{active} if it assumes value 1.

\Paragraph{Edge set.}
The set of edges $\hat\SE$ coupling the nodes is divided into subsets $\hat\SE^t$ corresponding to each time step $t\in\{1,\dots,T\}$, \ie $\hat\SE = \bigcup_{t=1}^{T}\hat\SE^t$.
In turn,
$\hat\SE^t = \hat\SE^t_\Smove \cup \hat\SE^t_\Sdiv \cup \hat\SE^t_\Sconf$, where
$\hat\SE^t_{\Smove}\subseteq \hat\SV^t_{\Sdet}\times\hat\SV^{t+1}_{\Sdet}\times \hat\SV^t_{\Strans}$ and
$\hat\SE^t_{\Sdiv}\subseteq\hat\SV^t_{\Sdet}\times(\hat\SV^{t+1}_{\Sdet})^2\times \hat\SV^t_{\Strans}$
are edges corresponding to possible moves and divisions of the cells between time steps $t$ and $t+1$, and $\hat\SE^t_{\Sconf}\subseteq 2^{\hat\SV^t_{\Sdet}}$ are the edges prohibiting the activation of conflicting (intersecting) segmentation hypothesis at time step $t$.
Note that for each node in $\hat\SV^t_{\Strans}$ there is exactly one incident edge.

\Paragraph{Coupling constraints.}
Let $(u,v,w) \in \hat\SE^t_\Smove$ be a possible move from time step $t$ to $t+1$ connecting nodes $u\in\hat\SV^t_\Sdet$ and $v\in\hat\SV^t_\Sdet$ via transition node $w\in\hat\SV^t_\Strans$.
We write $\SHatEdgeTrans uvw$ for such edges.
Then the set of corresponding coupling constraints is defined as
\begin{equation}\label{equ:standard_transition_constraint}
  \forall\, \SHatEdgeTrans uvw \in \hat\SE_{\Smove}: \quad x_w \le x_u \;\wedge\; x_w \le x_v\,,
\end{equation}
ensuring that if either of the hypothesis is deactivated~($x_u=0$ or $x_v=0$), the move is deactivated as well~($x_w=0$).
Analogously, we denote possible divisions $(u,v,v',w)\in \hat\SE^t_\Sdiv$ by $\SHatEdgeDiv{u}{v}{v'}{w}$, and obtain the following coupling constraints for divisions:
\begin{equation}\label{equ:standard_division_constraint}
 \forall\, \SHatEdgeDiv{u}{v}{v'}{w} \in \hat\SE_\Sdiv: x_w \!\le\! x_u \;\wedge\; x_w \!\le\! x_v \;\wedge\; x_w \!\le\! x_{v'}\,.
\end{equation}
For any detection node $v \in \hat\SV_\Sdet$ we denote by $\hat\SIn(v)$ the set of \emph{incoming} transitions, \ie~%
$\hat\SIn(v) := \{ w\in \hat\SV_\Strans\mid \exists\, u\colon \SHatEdgeTrans uvw \in \hat\SE_\Smove \text{ or } \exists\, u,v' \colon \SHatEdgeDiv{u}{v}{v'}{w} \in \hat\SE_\Sdiv \vee \SHatEdgeDiv{u}{v'}{v}{w} \in \hat\SE_\Sdiv \}$.
Likewise, for all $u\in\SV_\Sdet$ we define
$\hat\SOut(u) := \{ w\in \hat\SV_\Strans\mid\exists\, v\colon \SHatEdgeTrans uvw \in \hat\SE_\Smove \text{ or } \exists\, v,v'\colon \SHatEdgeDiv{u}{v}{v'}{w} \in \hat\SE_\Sdiv \}$
as the set of \emph{outgoing} transitions.

To guarantee that each hypothesis at time step $t$ is matched to at most one at time steps $t+1$ and $t-1$, \emph{uniqueness constraints} are introduced as follows:
\begin{equation}\label{equ:standard_uniqueness_constraint}
  \forall\, v\in\hat\SV_\Sdet\colon \sum_{w\in\hat\SIn(v)}\!\! x_w\le 1 \quad\text{ and }  \sum_{w\in\hat\SOut(v)}\!\!\! x_w\le 1\,.
\end{equation}
Finally, conflicting segmentation hypothesis are connected via similar constraints:
\begin{equation}\label{equ:standard_conflict_constraint}
  \forall\, c\in\hat\SE_\Sconf\colon \sum_{v\in c} x_v \le 1\,.
\end{equation}

\Paragraph{Objective function.}
Let $\hat\SX\subseteq\{0,1\}^{|\hat\SV|}$ be the set of binary vectors $x$ satisfying all coupling constraints defined by~\eqref{equ:standard_transition_constraint}-\eqref{equ:standard_conflict_constraint}.
Each coordinate $x_v$, $v\in\hat\SV$, is associated with a \emph{cost} $\theta_v\in \SR$ based on image data~(for segmentation hypothesis) and geometric priors (for transitions).
The goal of tracking is to find an assignment $x\in\hat\SX$ that minimizes the cost of the active binary variables, \ie which solves
\begin{equation}\label{equ:standard_ILP}
  \min_{\SVar\in\hat\SX} \; \SSP{\SCost}{x}\,,
\end{equation}
where $\SCost = (\SCost_v)_{v \in \hat\SV}$.
Problem~\eqref{equ:standard_ILP} is the standard ILP representation of the tracking problem.
In this form it is usually addressed~(see \eg~\cite{lou2011deltr,kausler2012discrete,schiegg2013conservation,schiegg2014graphical,jug2014optimal,jug2014tracking}) by off-the-shelf solvers like Gurobi~\cite{gurobi} or CPLEX~\cite{cplex}.
However, the run time and memory requirements of these solvers rapidly grow with the size of the input.
Moreover, even solving an LP relaxation of the above problem, \ie considering a vector in $[0,1]^{\smash{|\hat\SV|}}$ satisfying~\eqref{equ:standard_transition_constraint}-\eqref{equ:standard_conflict_constraint}, requires a significant time using standard solvers, as they are based on simplex or interior point methods with a super-linear iteration complexity.
Note that first order subgradient-based methods perform even slower than the standard solvers~\cite{kappes-2015-ijcv}.

\section{Our decomposable representation}
\label{sec:decomposition}

Efficiency of large-scale approximate optimization methods heavily depends on the problem decomposition used to build a dual problem.
A good decomposition should contain small number of easily tractable subproblems.
Consider a trivial decomposition of~\eqref{equ:standard_ILP}, when every binary variable corresponds to a separate subproblem.
It would satisfy the tractability condition, but the large number of subproblems would significantly slow down the optimization.
Therefore, below we give an alternative representation of the problem~\eqref{equ:standard_ILP}, which leads to a natural decomposition with a much smaller number of easily tractable subproblems.

\Paragraph{Lagrange decomposition idea.}
Assume we want to minimize a function $F(x)$ representable as $F(x)=F_1(x)+F_2(x)$.
The Lagrange decomposition~\cite{guignard1987lagrangeana,guignard1987lagrangeanb,guignard2003lagrangean,sontag2011introduction} duplicates the variable $x$ and introduces the equality constraint $x_1=x_2$, \ie $\min_{x} F(x) = \min_{x_1,x_2\colon x_1=x_2}\bigl(F_1(x_1) + F_2(x_2)\bigr)$.
Dualization of the constraint $x_1=x_2$ leads to the \emph{Lagrange dual problem}, which forms a lower bound for the original problem:
\begin{gather}
  \min_{x} F(x) \ge \max_{\lambda}\min_{x_1,x_2}\bigl(F_1(x_1)+F_2(x_2)+ \SSP{\lambda}{x_1-x_2}\bigr) \nonumber\\
  \!\!\!\! =\!\max_{\lambda}\Bigl[\min_{x_1}\bigl(\!\!\:F_{\!\!\:1}\!\!\;(x_1\!\!\;) \!+\! \SSP{\!\!\;\lambda}{x_1}\!\!\;\bigr) \!\!+\! \min_{x_2}\bigl(\!\!\;F_{\!\!\:2}\!\!\;(x_2\!\!\;) \!-\! \SSP{\!\!\;\lambda}{x_2}\!\!\;\bigr)\!\!\;\Bigr]\!.\!\!
  \label{equ:lagrange_decomposition_idea}
\end{gather}
Tightness of the lower bound as well as the efficiency of its maximization depend on the decomposition of $F$ into $F_1$ and $F_2$.
Ideally, the minimization subproblems over $x_1$ and $x_2$ are solvable in closed form, and the coupling constraint $x_1=x_2$ is only violated in a small subset of coordinates of the subproblem minima.
The dual vector $\lambda$ allows to reweight the functions associated with the duplicated variables during optimization in order to reduce violations of coupling constraints.

\Paragraph{Decomposed graph.}
We will apply the Lagrange decomposition idea to problem~\eqref{equ:standard_ILP}.
To this end, we first duplicate all binary variables and then regroup them.
Each group corresponds to a new graph node.
This leads to considerably less nodes.
Although each node is associated with a non-binary variable, its minimal value can still be efficiently found.
All coupling constraints turn into simple equalities as in the general scheme~\eqref{equ:lagrange_decomposition_idea}.

\Paragraph{Graph structure.}
Our graph $\SG=(\SV,\SE)$, see Figure~\ref{fig:decomposition}~(c), contains only two types of nodes: detection and conflict nodes.
The transition variables are duplicated (tripled for divisions) and their copies are assigned to the corresponding detection nodes.
Since each detection corresponds to a large number of transitions, this significantly decreases the graph size.
The detection variables are duplicated as well and their copies are assigned to the detection and conflict nodes.
Below we give the formal definitions.

The set of graph nodes is defined as $\SV=\SV_\Sdet\cup\SV_\Sconf$, where $\SV_\Sdet := \hat\SV_\Sdet$ and $\SV_\Sconf := \hat\SE_\Sconf$, \ie the detection nodes and conflict edges in the standard model correspond to detection and conflict nodes in our model.
As in the standard model, $\SV_\Sdet^t$ and $\SV_\Sconf^t$ denote the nodes at time step $t\in \{1,\ldots,T\}$.

Each edge in the edge set $\SE = \SE_\Smove \cup \SE_\Sdiv \cup \SE_\Sconf$ corresponds to either a transition or conflict.
The transition edges divide into $\SE_{\Smove}:=\hat\SE_{\Smove}$, ${\SE_\Smove \subseteq(\SV_\Sdet)^2}$, corresponding to moves, and $\SE_{\Sdiv}:=\hat\SE_{\Sdiv}$, $\SE_\Sdiv \subseteq (\SV_\Sdet)^3$, corresponding to divisions.
As for the standard model, we will denote an edge $(u,v)\in \SE_\Smove$ by $\SEdgeTrans uv$, and an edge $(u,v,w) \in \SE_\Sdiv$ by $\SEdgeDiv uvw$.

Conflict edges $(u, c)$, denoted by $\SEdgeConf uc$, are introduced between any detection node $u \in \SV_\Sdet$ and conflict node $c \in \SV_\Sconf$ as soon as $\hat u \in \hat c$ for the corresponding $\hat u\in\hat\SV_\Sdet$ and $\hat c\in\hat\SE_{\Sconf}$.
Note, $\SE_\Sconf \subseteq \SV_\Sdet \times \SV_\Sconf$.
When considering a conflict node $c \in \SV_\Sconf$, we also use $c$ to refer to all detection nodes that are part of the conflict, \ie $u \in c$ if and only if $\SEdgeConf uc \in \SE_\Sconf$.
Furthermore, for any $u\in\SV_\Sdet$ we define $\Sconf(u) := \{ \SEdgeConf {u'}c \in \SE_\Sconf \mid u = u' \} $ as the set of all conflicts concerning $u$.

\Paragraph{Detection variables.}
As noted above, nodes of the graph $\SG$ correspond to variables having more than two states.
These states are represented by binary vectors.

To define the state space of the detection variables we first introduce the sets $\SIn(u)$ and $\SOut(u)$, corresponding to $\hat\SIn(\hat u)$ and $\hat\SOut(\hat u)$ in $\hat\SG$.
For any detection node $u' \in \SV_\Sdet$ we denote by $\SIn(u')$ the set of all incoming transitions, \ie~%
$\SIn(u') := \{ \SEdgeTrans uv \in \SE_\Smove \mid v = u' \} \cup \{ \SEdgeDiv uvw \in \SE_\Sdiv \mid v = u' \text{ or } w = u' \}$.
Analogously, we define $\SOut(u') := \{ \SEdgeTrans uv \in \SE_\Smove \mid u = u' \} \cup \{ \SEdgeDiv uvw \in \SE_\Sdiv \mid u = u' \}$ as the set of all outgoing transitions.

Consider $u \in \SV_\Sdet$.
The set of states $\SX_u$, which models whether the detection $u$ is active, and if so, which incoming and outgoing edge is active, is defined as
\begin{equation}\label{equ:detection_factor}
  \SX_u \!=\! \left\{\!
      \left(
      \begin{aligned}
        x_\Sdet &\in \{0,1\}\\[-3pt]
        x_\SIn  &\in \{0,1\}^{|\SIn(u)|}\\[-3pt]
        x_\SOut &\in \{0,1\}^{|\SOut(u)|}
      \end{aligned}
      \right)
    \middle|\!
      \begin{array}{l}
        \SSP{\SOne}{x_\Sin} \leq x_\Sdet,\\[.3em]
        \SSP{\SOne}{x_\Sout} \leq x_\Sdet
      \end{array}
  \!\!\right\}\mathrlap{.}
\end{equation}
The scalar products $\SSP{\SOne}{x_\Sin}$ and $\SSP{\SOne}{x_\Sout}$ express the number of activated incoming and outgoing transitions.
Note that they can only be non-zero if the detection is active, \ie if $x_\Sdet = 1$.
Below we use $x_\Sin(e)$ for any incoming edge $e \in \Sin(u)$ to refer to the value of this edge in the current state.
Analogously, we use $x_\Sout(e)$ for edges $e \in \Sout(u)$.
Recalling the standard model, for any $w \in \hat\SV_{\Strans}$ associated with $\SHatEdgeTrans uvw \in \hat\SE_\Smove$ ($\SHatEdgeDiv u{v}{v'}w \in \hat\SE_\Sdiv$) the binary variable $x_w$ is split into two (three) variables, one belonging to $\SOut(u)$ and another to $\SIn(v)$ (and $\SIn(v')$).

With each detection $u \in \SV_\Sdet$ we associate a cost vector $\theta_u = ( \theta_\Sdet, \theta_\SIn, \theta_\SOut )$ consisting of the cost $\theta_\Sdet \in \SR$ for activating the detection, and costs $\theta_\SIn \in \SR^{|\SIn(u)|}$ and $\theta_\SOut \in \SR^{|\SOut(u)|}$ associated with the incoming and outgoing edges.
$\theta_\Sdet = \theta_{\hat{u}}$, where $\hat{u}\in\hat\SV_\Sdet$ is the corresponding detection node in the standard model.
$\SCost_\Sin$ and $\SCost_\Sout$ are obtained by splitting the given transition costs between incoming and outgoing variables, \ie~%
$\theta_\SIn (e) = \smash{\frac{1}{|e|}} \theta_{\hat{w}}$ for all $e \in \SIn (u)$, and, analogously, $\theta_\SOut (e) = \smash{\frac{1}{|e|}} \theta_{\hat{w}}$ for all $e \in \SOut (u)$, where $\hat{w} \in \hat\SV_\Strans$ is the transition node corresponding to $e$ in the standard model.
So each admissible state $x = (x_\Sdet, x_\SIn, x_\SOut) \in \SX_u$ has a linear cost
$\SSP{\theta_u}{x} = \SSP{\theta_\Sdet}{x_\Sdet} + \SSP{\theta_\SIn}{x_\SIn} + \SSP{\theta_\SOut}{x_\SOut}$.

\Paragraph{Conflict variables.}
Let $c \in \SV_\Sconf$.
The associated set of states $\SX_c$, that models which of the conflicting detections, if any, is active, can be written as
\begin{align}\label{equ:conflict_factor}
  \SX_c &=
    \bigl\{
      x \in \{ 0,1 \}^{|c|} \mid \SSP{\SOne}{x} \leq 1
    \bigr\} .
\end{align}
For any detection node $u \in c$ we write $x(u)$ to refer to the value of this detection in the current state.
With~$c$ we associate a cost vector $\theta _c \in \SR^{|c|}$.
These costs are initially zero, but may change during optimization.
So each admissible state $x \in \SX_c$ has a linear cost $\langle \theta_c, x \rangle$.

\Paragraph{Coupling constraints.}
The semantics of each single move, division and detection activation is split between the states of multiple nodes in our problem graph.
Consider a move $\SEdgeTrans uv$ from detection node $u$ to $v$.
To obtain a consistent solution, we require moves to be consistent in $u$ and $v$, i.e. $\SVarFactorOut{u}(\SEdgeTrans uv) = \SVarFactorIn{v}(\SEdgeTrans uv)$ for any feasible combination of states $\SVarFactor u \in \SX_u$, $\SVarFactor v \in \SX_v$.
Analogous considerations for divisions and conflicts result in the following coupling constraints for $\SG$:
\begin{align}
  &\forall\, e = \SEdgeTrans uv \in \SE_\Smove\colon
  &\SVarFactorOut{u}(e) &= \SVarFactorIn{v}(e), \nonumber\\
  &\forall\, e = \SEdgeDiv uvw \in \SE_\Sdiv\colon
  &\SVarFactorOut{u}(e) &= \SVarFactorIn{v}(e), \nonumber\\
  &\forall\, e = \SEdgeDiv uvw \in \SE_\Sdiv\colon
  &\SVarFactorOut{u}(e) &= \SVarFactorIn{w}(e), \nonumber\\
  &\forall\, \SEdgeConf uc \in \SE_\Sconf\colon
  &\SVarFactorDet{u} &= \SVarFactorConf{c}(u) \, . \label{equ:coupling_constraints}
\end{align}
The set of all state assignments satisfying all coupling constraints can then be written as
\begin{align*}
  \SX & = \bigl\{
      \bigl(
        \{\SVarFactor v \in \SX_v \}_{v\in\SV_\Sdet},
        \{\SVarFactor c \in \SX_c \}_{c\in\SV_\Sconf}
      \bigr)
    \,\bigm|\, \text{\eqref{equ:coupling_constraints}}
  \bigr\} .
\end{align*}

\Paragraph{Minimization problem.}
Our graph decomposition naturally gives rise to the minimization problem
\begin{equation}
  \min_{x \in \SX}
    \biggl[
      E(\theta, x)
      \!:=\!
      \sum_{\mathclap{u \in \SV_\Sdet}}  \SSP{\SCostFactor u}{\SVarFactor u} +
      \sum_{\mathclap{c \in \SV_\Sconf}} \SSP{\SCostFactor c}{\SVarFactor c}
    \biggr] .
  \label{equ:energy_minimization}
\end{equation}
The goal is to find an optimal state assignment that satisfies all coupling constraints given costs $\theta$.

\Paragraph{Dualization of coupling constraints.}
Let us return to the general idea of the Lagrange decomposition~\eqref{equ:lagrange_decomposition_idea}.
Assume $F_1(x)=\SSP{\theta_1}{x}$.
Then $F_1(x)-\SSP{\lambda}{x} = \SSP{\theta_1-\lambda}{x}$.
Similarly, if $F_2(x)=\SSP{\theta_2}{x}$, then $F_2(x)+\SSP{\lambda}{x} = \SSP{\theta_2+\lambda}{x}$.
In other words, $\lambda$ shifts the costs between parts of the decomposed problem.
Since the value of the objective $F(x)=\SSP{\theta_1-\lambda}{x_1} +\SSP{\theta_2+\lambda}{x_2}$ remains the same for any value of $\lambda$ if $x_1=x_2$, this is also referred to as a \emph{reparametrization} of the problem.

Dualizing all coupling constraints~\eqref{equ:coupling_constraints} in problem~\eqref{equ:energy_minimization} results in the following reparametrized cost vectors:

\Paragraph{Reparametrization.}
A \emph{reparametrization} is a vector $\SRepa \in \Lambda:=\SR^{|\SE_\Smove|+2|\SE_\Sdiv|+|\SE_\Sconf|}$.
Its coordinates will be indexed with edges of the graph $\SG$.
That is, $\SRepa(e)\in\SR$ is the dual variable corresponding to the constraint $x_{u,\Sin}(e)=x_{v,\Sout}(e)$ if $e = \SEdgeTrans uv \in\SE_{\Smove}$, and $\SVarFactorDet{u}=\SVarFactorConf{c}(u)$ if $e=\SEdgeConf uc \in\SE_{\Sconf}$.
The only exception are divisions, since two constraints must be dualized for each $e = \SEdgeDiv uvw \in \SE_\Sdiv$, namely $\SVarFactorOut{u}(e) = \SVarFactorIn{v}(e)$ and $\SVarFactorOut{u}(e) = \SVarFactorIn{w}(e)$, \cf~\eqref{equ:coupling_constraints}.
The corresponding dual variables are denoted as $\SRepa_{v}(e)$ and $\SRepa_{w}(e)$ respectively.
The \emph{reparametrized costs} $\SCostRepa$ are defined as
\begin{flalign*}
  &
  \begin{aligned}
    &\forall\, c\in\SV_\Sconf \,\forall\, u\in c \colon &
    \SCostRepaFactor c(u) &:=
      \SCostFactor c(u) + \SRepa(\SEdgeConf uc), &&
    \\
    &\forall\, u\in\SV_\Sdet \colon &
    \SCostRepaFactorDet u  &:=
      \SCostFactorDet u - \sum_{\mathclap{e \in \Sconf (u)}} \SRepa(e),
  \end{aligned}
  \\[-0.8em]
  &\forall\, u\in\SV_\Sdet \,\forall\, e \in \SIn(u) \colon
  \\
  &\quad\SCostRepaFactorIn u(e) :=
      \begin{cases}
        \SCostFactorIn u(e) + \SRepa(e), & \!\!\!e \in \SE_\Smove \\
        \SCostFactorIn u(e) + \SRepa_u(e), & \!\!\!e \in \SE_\Sdiv
      \end{cases},
  \\
  &\forall\, u\in\SV_\Sdet \,\forall\, e \in \SOut(u) \colon &
  \\
  &\quad\SCostRepaFactorOut{u}(e) :=
      \begin{cases}
        \SCostFactorOut u(e) \!-\! \SRepa(e),              & \!\!\!e \in \SE_\Smove \\
        \SCostFactorOut u(e) \!-\! \sum\limits_{\mathclap{v'\in\{v,w\}}} \SRepa_v(e), & \!\!\!e\! =\! \SEdgeDiv uvw \!\in\! \SE_\Sdiv\!
      \end{cases}.
\end{flalign*}
Each element of $\SRepa$ shifts the cost between two copies of variables in different subproblems coupled by an edge in~$\SE$.
This shift does not influence the optimization objective as long as the coupling constraints~\eqref{equ:coupling_constraints} hold:

\begin{proposition}
  $\forall x \in \SX,\  \SRepa\in \Lambda\colon E(\SCost, x) = E(\SCost^\SRepa, x).$
\end{proposition}

Our dual is built similarly to the general scheme~\eqref{equ:lagrange_decomposition_idea}:
\begin{proposition}
  Dualizing all coupling constraints~\eqref{equ:coupling_constraints} in the objective~\eqref{equ:energy_minimization} yields the \emph{Lagrange dual problem} $\max_{\SRepa\in\Lambda} D(\SRepa)$, where
  \begin{equation}\label{equ:dual_objective}
    D(\SRepa) :=
      \sum_{\mathclap{u \in \SV_\Sdet}}  \; \min_{x_u \in \SX_u} \SSP{\SCostRepaFactor u}{\SVarFactor u} +
      \sum_{\mathclap{c \in \SV_\Sconf}} \; \min_{x_c \in \SX_c} \SSP{\SCostRepaFactor c}{\SVarFactor c}
    \,.
  \end{equation}
\end{proposition}%
Obviously, $D(\SRepa)$ is concave and piecewise linear, \ie non-smooth.
By construction, $\max_{\SRepa\in\Lambda} D(\SRepa) \le \min_{\SVar\in\SX} E(\SCost, \SVar)$.
The dual objective $D(\SRepa)$ is a sum of small-sized minimization problems.
Due to the structure of~\eqref{equ:detection_factor} and~\eqref{equ:conflict_factor}, each subproblem related to a graph node can be solved in linear time.
As we show in the supplement, the maximization of the dual~{\eqref{equ:dual_objective}} yields the same value as the natural LP relaxation of~{\eqref{equ:energy_minimization}}.

\section{Dual block-coordinate ascent (BCA)}
\label{sec:dbca}

\input{algorithm_dual}

In order to maximize~\eqref{equ:dual_objective}, we developed an algorithm based on the BCA principle, as such methods perform competitively for similar relaxations of large-scale combinatorial problems.
These techniques received a lot of attention in connection with the local polytope relaxation of the discrete energy minimization problem~\cite{kappes-2015-ijcv}.
BCA methods like TRW-S~\cite{kolmogorov2006convergent}, SRMP~\cite{kolmogorov2015new} or the recently proposed DMM~\cite{Shekhovtsov-DMM-16} and MPLP++~\cite{tourani2018mplp} notably outperform off-the-shelf solvers as well as dedicated subgradient-based methods.
However, they are inapplicable in our case, due to a substantially different problem structure.
The work~\cite{swoboda2017dual} partially fills this gap by proposing a general framework for constructing dual BCA algorithms for a substantial subclass of combinatorial problems, which covers our problem as well.
However, our experiments with the framework~\cite{swoboda2017dual} did not lead to an improvement over Gurobi.
Therefore, we constructed a new algorithm, which resembles the local polytope technique~\cite{kolmogorov2006convergent,kolmogorov2015new,Shekhovtsov-DMM-16,tourani2018mplp,savchynskyy2019discrete} and at the same time uses the results of~\cite{swoboda2017dual} to guarantee monotonicity of the dual improvement.\footnote{In particular, our updates are \emph{admissible}~\cite[Lem.1]{swoboda2017dual}, but do not satisfy the maximality condition~\cite[eq.(15)]{swoboda2017dual}.}

\Paragraph{Dual BCA algorithm.}
Algorithm~\ref{alg:dual} is a realization of the BCA principle for the Lagrange dual~\eqref{equ:dual_objective} and guarantees its monotonous improvement.
It contains four types of updates (also referred to as passing \emph{messages}).
The first two, $\SMsgFromFact u$ and $\SMsgFromConf c$, called \emph{conflict updates}, reweight the detection variable costs~$\theta_{u,\Sdet}$ and conflict variable costs~$\theta_c(u)$.
The second two, $\SMsgRight u$ and $\SMsgLeft u$, called \emph{transition updates}, reweight costs $\theta_{u,\Sin}$ and $\theta_{v,\Sout}$ across consecutive time steps.
All update vectors $\SMsgFromFact u$, $\SMsgFromConf c$, $\SMsgRight u$ and $\SMsgLeft u$ are in~$\Lambda$, with zero assigned to the unaffected coordinates.
We process the time steps sequentially.
We first perform all conflict updates within the current time step, then we propagate the costs to the next time step via transition updates.

\Paragraph{Conflict updates.}
In the dual problem $D(\SRepa)$ a detection will favor activation ($\SVarDet \!=\! 1$)
if its locally minimal state has negative cost.
Ideally, only a single detection connected to a particular conflict node
should be active.
Otherwise, a coupling constraint on at least one of the edges in $\SE_\Sconf$ or the constraint in $\SX_c$ is violated, see~\eqref{equ:conflict_factor},~\eqref{equ:coupling_constraints}.
The following updates encourage agreement of the local minimizers of the conflict nodes and the associated detection nodes.
We define for all $u \in \SV_\Sdet$, $e\in\Sconf (u)$:
\begin{align*}
  \SMsgFromFact u(e) &:=
	\min\limits_{x\in\SX_u \colon x_\Sdet = 1} \frac{\SSP{\SCost _u}{x}}{|\Sconf (u)|} .
\end{align*}
Intuitively, $\SMsgFromFact u$ redistributes as much of the cost as possible to the connected conflict nodes while preserving the locally optimal state in the detection node.
Similarly, we define for all $c \in \SV_\Sconf$, $u \in c$:
\begin{align*}
  \SMsgFromConf c(e) &:=
      - \SCost _c(u) + \frac{1}{2} \bigl[ \SSP{\SCost _c}{z_c^\star} + \SSP{\SCost _c}{z_c^{\star\star}} \bigr],
\end{align*}
where $e \!=\! \SEdgeConf uc$, and $z_c^\star \!=\! \argmin_{x \in \SX_c} \! \SSP{\SCost _c}{x}$ is the best and $z_c^{\star\star} \!=\! \argmin_{\smash{x \in \SX_c\setminus\{z_c^\star\}}} \!\SSP{\SCost _c}{x}$ the second-best state.
$\SMsgFromConf c$ shifts the cost back such that only the most promising detection ends up with a negative activation cost.

\Paragraph{Transition updates.}
To propagate information across time steps, we introduce the dual updates $\SMsgRight u$ and $\SMsgLeft u$.
We first define for all $u \in \SV_\Sdet$:
\begin{align*}
  x_u^\star & := \;\argmin_{\mathclap{x \in \SX_u\colon x_\Sdet = 1}}\; \SSP{\SCost_u}{x}, &
  y_u^\star & := \;\argmin_{\mathclap{\substack{x \in \SX_u\colon x_\Sdet = 1, \\ x_\SIn \neq x^\star_{u,\SIn},\; x_\SOut \neq x^\star_{u,\SOut}}}}\; \SSP{\SCost_u}{x}.
\end{align*}
Intuitively, $x_u^\star$ is the best state of $u$ under the assumption that $u$ is active, while $y_u^\star$ is the best state differing from $\SVar_u^\star$ in the incoming and outgoing edge.

Similar to the conflict updates, the transition updates propagate as much information as possible by setting all the outgoing respectively incoming costs to the same value.
We define for all $u \in \SV_\Sdet$, $e \in\SOut (u)$:
\begin{gather*}
  \SMsgRight u(e) := \!\!
    \min\limits_{\substack{x \in \SX _u\colon \\ x_\SOut (e) = 1}} \!\!\!
        \SSP{\SCost _u}{x} - \SMinMarFactorOut u
      ,\, \text{if}\ e \in \SE_\Smove \\
  (\SMsgRight u)_v(e) :=  \tfrac{1}{2} \Bigl[\!
          \min\limits_{\substack{x \in \SX _u\colon \\ x_\SOut (e) = 1}} \!\!\!
          \SSP{\SCost _u}{x} \!-\! \SMinMarFactorOut u
        \Bigr]
      ,\, \text{if}\ e=\SEdgeDiv uvw
\end{gather*}
\smash{where}%
\vspace{-.4em}
\begin{align*}
  \SMinMarFactorOut u & := \min \Bigl\{ 0, \tfrac{1}{2}\bigl[ \SSP{\SCost _u}{x_u^\star} + \SSP{\SCost _u}{(1, x^\star_{u,\SIn}, y^\star_{u,\SOut})} \bigr] \Bigr\},
\end{align*}
is either 0 or the mean value between the cost of~$x_u^\star$ and the next-best state with a different outgoing edge.

Similarly, for the updates in the opposite direction, we define for all $v \in \SV_\Sdet$, $e \in\SIn (v)$:
\begin{equation}\label{equ:transition_update_left}
  \SMsgLeft v(e) :=
    - \min\limits_{\mathclap{x \in \SX _v, x_\SIn (e) = 1}} \SSP{\SCost _v}{x}
      + \SMinMarFactorIn v,
    \text{ if } e\in\SE_\Smove\,.
\end{equation}
Otherwise, if $e\in\SE_\Sdiv$, $(\SMsgLeft v)_v(e)$ is assigned the right-hand-side of~\eqref{equ:transition_update_left}.
Here
\begin{align*}
  \SMinMarFactorIn v & := \min \Bigl\{ 0, \tfrac{1}{2}\bigl[ \SSP{\SCost _v}{x_v^\star} + \SSP{\SCost _v}{(1, y^\star_{v,\SIn}, x^\star_{v,\SOut})} \bigr] \Bigr\},
\end{align*}%
is either 0 or the mean value between the cost of $x_v^\star$ and the next-best state with a different incoming edge.

\begin{proposition}
  Dual updates
  $\SMsg \in \{ \SMsgLeft u, \SMsgRight u, \SMsgFromFact u \mid u \in \SV_\Sdet \} \cup \{ \SMsgFromConf c \mid c \in \SV_\Sconf \}$
  monotonically increase the dual function, \ie~%
  $\forall\lambda\in\Lambda\colon D(\lambda) \leq D(\lambda+\SMsg)$.
\end{proposition}

Since the optimal dual value is bounded from above by the optimum of~\eqref{equ:standard_ILP}, the monotonicity implies convergence of the sequence of dual values.
The limit value of the sequence, though, need not be the dual optimum.
This is a well-known property of block-coordinate-ascent methods, which may get stuck when applied to non-smooth functions.

\input{algorithm_primal}

\section{Primal heuristics}
\label{sec:primal_heuristic}

For solving the tracking problem~\eqref{equ:energy_minimization} computing a reparametrization $\SRepa$ is not enough.
The main goal is to obtain a feasible primal assignment $x \in \SX$ corresponding to a low objective value.
Generally, this is non-trivial, since solutions to the node subproblems are usually inconsistent, even if an optimal dual solution $\SRepa^\star = \argmax_{\SRepa'\in\Lambda} D(\SRepa')$ is considered.
As mentioned in Section~\ref{sec:introduction}, existing techniques either do not handle overlapping segmentation hypotheses or do not allow for cell divisions.
Below, we propose a new primal heuristics that handles the considered general case and produces high-quality feasible primal assignments even for non-optimal dual vectors~$\SRepa$.

\Paragraph{Temporal direction.}
Similar to the dual optimization algorithm our primal heuristics works in both temporal directions.
For the sake of simplicity we restrict ourselves to the explanation of the forward direction.
After obtaining the primal assignments for each direction, we keep the best of the two.

\Paragraph{Incremental estimates.}
We estimate primal solution for each time frame sequentially one by one, starting from $t=0$ and ending with $t=T$.
Hence, we assume that primal assignments $x_u \in \SX_u$ for $u \in \SV_\Sdet^{\smash{t'}}$, $t' \in \{1, \ldots, t-1\}$, are available when applying Algorithm~\ref{alg:primal} for time step~$t$.

\input{figures/plots}

\input{tables/results}

\Paragraph{Conflict resolution.}
Before we look at edges that connect different time steps, we first resolve the conflicts within the current time step $t$.
Overall, we want to activate the most promising detection hypotheses while still obeying all coupling constraints in $\SE_\Sconf$.
We score the individual detections $v \in \SV_\Sdet$ by their locally cost-optimal state conditioned on actually activating~$v$,~i.e.~$s(v) := \min_{x\in\SX_v \colon x_\Sdet=1} \SSP{\SCostRepaFactor v}{x}$.
To resolve the conflicts of the current time step $t$, we solve with Gurobi the \emph{weighted set packing problem} of all detections in $t$:
\begin{equation}
    \min_{\mu\in\{0,1\}^{|\SV_\Sdet^t|}} \SSP{s}{\mu}
    \quad
    \text{s.t. } \sum_{v\in c} \mu_v \le 1 \;\;\;\forall\, c \in \SV_\Sconf^t\;,
  \label{equ:conflict_resolution}
\end{equation}
where $s=(s(v))_{v\in\SV_\Sdet^t}$.
Note that even though the problem~{\eqref{equ:conflict_resolution}} is NP-hard, its encountered instances are small and can be solved almost instantly.
In comparison to cheaper approaches, using~{\eqref{equ:conflict_resolution}} yields significantly better tracking solutions.

\Paragraph{Transition assignments.}
After resolving all conflicts for the current time step $t$, we look for a consistent assignment for $\SVarFactorIn v$ for all $v\in\SV^{\smash{t}}_\Sdet$ with $\mu^{\smash{\star}}_v = 1$,
where~$\mu^{\smash{\star}}$ is a solution of~\eqref{equ:conflict_resolution}.
Ideally, we would like to set them to a locally optimal state given their current cost $\SCostRepaFactor v$.
Unfortunately, this will in general violate coupling constraints in $\SE_\Smove^{t-1}$ and $\SE_\Sdiv^{t-1}$.
Therefore, we use a greedy approach and process all possible nodes, \ie all $v \in \SV^{\smash{t}}_\Sdet$ with $\mu^\star_v = 1$, sequentially.
While optimizing $\SVarFactorIn v$ for a given node~$v$ we ignore all options that would violate coupling constraints with already fixed nodes.
Obviously, such a procedure heavily depends on the node ordering.
As we want to address the most promising states first, we order all nodes by their score~$s(v)$ starting with the state with the best, \ie lowest, score.

\Paragraph{Propagation.}
When $\SVarFactorIn v$ is assigned we propagate this across incident coupling constraints in $\SE_\Smove^{t-1}\cup\SE_\Sdiv^{t-1}$.
Eventually, this results in all variables $\SVarFactorOut u$ being set for all nodes $u \in \SV^{t-1}_\Sdet$ of the previous time step.
This procedure leads to a consistent variable assignment up to and including time step $t$.

Continuing the procedure for all time steps
leads to a consistent variable assignment for the whole problem.
As the quality of the assignment is expected to improve with improvement of the dual, we run primal heuristics repeatedly each $25$ dual iterations.

\vspace{0pt plus 1pt}
\section{Experimental evaluation}
\label{sec:experiments}

We evaluate the performance of our solver on cell and nucleus tracking datasets from different biomedical research projects.
As our contributions are exclusively concerned with the optimization of the problem instances, we are not discussing any aspects of potential model mismatch.
We used reasonable segmentations obtained by different segmentation routines and selected appropriate costs for the different tracking events without excessive fine-tuning (see supplement).
Existing methods that allow for overlapping segmentation hypotheses and cell divisions~\cite{jug2014optimal,jug2014tracking,kaiser2018moma,schiegg2014graphical} offload the optimization to off-the-shelf ILP solvers like Gurobi~\cite{gurobi},
which we use as baseline for our comparison.

\Paragraph{Datasets and tracking instances.}
In total, we use 10 problem instances from 3 biomedical data domains, see supplement for additional information.

\emph{Drosophila embryo data}:
We have one problem instance for tracking nuclei in a developing \emph{Drosophila} embryo.
The tracking model consists of 252~frames, each containing about 320~detection hypotheses.
With about 160~actual objects per time step, we typically observe two conflicting hypotheses per real object in the data.

\emph{Flywing data}:
We use 3 problem instances for tracking membrane-labelled cells in developing \emph{Drosophila} flywing tissue.
Two of these consist of 100~frames, each containing about 2\,100~detection hypotheses.
The third instance is larger, consisting of 245~frames with more than 3\,300~detection hypotheses each.
In contrast to the embryo data, the segmentation hypotheses in these problem instances are very dense, leading to considerably larger sets of (transitively) conflicting detections.

\emph{Cell Tracking Challenge (CTC) data}:
Finally, we use the publicly available cell tracking datasets \emph{Fluo-C2DL-MSC}, \emph{Fluo-N2DH-GOWT1}, and \emph{PhC-C2DL-PSC}~\cite{Ulman2017ctc} to evaluate our solver.
The CTC data consists of~48, 92, and 426~frames, where each frame contains on average~88, 186, and 1\,400~detection hypotheses, respectively, with conflict set sizes of about 10, 7, and 3.
Each dataset consists of two time-lapse movies, allowing us to generate six tracking instances.

\Paragraph{Evaluation criteria.}
We evaluate the performance using two metrics, namely,
total \emph{run time} and overall \emph{memory consumption}.
To ensure a fair run time comparison between our solver and Gurobi, we measure not only the time it takes Gurobi to compute the optimal solution, \emph{time (opt.)}, but also the time Gurobi needs to surpass the quality of our best primal solution, \emph{time (equ.)}.
Additionally, we compute the \emph{relative error}
$\frac{|E(\theta, x) - E(\theta, x^\star)|}{|E(\theta, x^\star)|}$
of our final assignment $x$ with respect to the optimal solution $x^\star$.
We finally also compute the \emph{TRA}~\cite{Ulman2017ctc} score, a commonly used tracking scoring function.
TRA values are in $[0, 1]$, where 1 means that the final assignment is identical to the reference solution, while 0 occurs if the compared solutions have nothing in common.
We use TRA to compare our tracking results to the optimal solution obtained with Gurobi.
Times show the median results of 5 single-threaded runs on an Intel i7-4770 3.40GHz CPU.

\Paragraph{Results and conclusions.}
In Table~\ref{tab:results} we show the results for our solver and Gurobi on all instances.
Especially for larger problems, \ie the \emph{PhC-C2DL-PSC} and \emph{flywing} instances, our solver obtains near-optimal solutions in a fraction of the time (sped-up by factor between 3 and 60).
Due to the compact decomposition, our solver consistently requires considerably less memory than Gurobi (up to 17x).
Figure~\ref{fig:plots} shows that our solver converges to small relative errors after only a few iterations.
The high quality of our solutions is confirmed by the TRA scores.
Compared to existing techniques our solver is scalable and quickly provides high quality solutions even for large-scale real-world problems, so far beeing practically intractable.
Therefore it is even applicable in low-latency settings like user-driven proofreading of automated tracking results.

\Paragraph{Acknowledgements.}
We thank the Centre for Information Services and High Performance Computing~(ZIH) at TU Dresden for generous allocation of HPC resources~(project HPDLF).
We acknowledge Romina Piscitello Gomez and Suzanne Eaton from MPI-CBG for sharing flywing data and Hernan Garcia lab from UC Berkeley for sharing the drosophila dataset.
This work was supported partly by the German Research Foundation~(DFG SA 2640/1-1, ``ERBI''), the European Research Council~(ERC Horizon 2020, grant 647769) and the German Federal Ministry of Research and Education (BMBF, code 01IS18026C,~``ScaDS2'').

\bibliography{paper}

\input{supplement}

\end{document}

%% file: preamble.tex
\usepackage[utf8]{inputenc}
\usepackage[T1]{fontenc}
\usepackage[english]{babel}

\usepackage[numbers,square,sort]{natbib}
\usepackage{amsfonts}
\usepackage{amssymb}
\usepackage{amsthm}
\usepackage{booktabs}
\usepackage{nicefrac}
\usepackage{etoolbox}
\usepackage{url}

\usepackage{stmaryrd}

\usepackage[ruled,vlined]{algorithm2e}

\usepackage[dvipsnames]{xcolor}
\usepackage{caption}
\usepackage{subcaption}
\usepackage{mathtools}
\usepackage{dsfont}

\usepackage{microtype}
\usepackage[colorlinks,citecolor=green!80!black,urlcolor=blue]{hyperref}

\input{notation}

\newcommand\Paragraph[1]{\textbf{#1}}

\newtheoremstyle{plain}%
  {\parskip}
  {0pt}
  {\itshape}
  {}
  {\bfseries}
  {.}
  {.5em}
  {}

\theoremstyle{plain}
\newtheorem{proposition}{Proposition}
\newtheorem{proposition-supplement}{Proposition}
\newtheorem{lemma}{Lemma}
\newtheorem{corollary}{Corollary}

\def\baselinestretch{0.975}

%% file: notation.tex
\newcommand{\ie}{i.\,e.\xspace}
\newcommand{\eg}{e.\,g.\xspace}
\newcommand{\cf}{cf.\xspace}

\DeclareMathOperator*\argmin{arg\,min}
\DeclareMathOperator*\argmax{arg\,max}

\newcommand\SR{\mathbb{R}}
\newcommand\SX{\mathcal{X}}
\newcommand\SV{\mathcal{V}}
\newcommand\SE{\mathcal{E}}
\newcommand\SG{\mathcal{G}}
\newcommand\SOne{\mathds{1}}

\newcommand\Sdet{{\mathsf{det}}}
\newcommand\Strans{{\mathsf{trans}}}
\newcommand\Smove{{\mathsf{move}}}
\newcommand\Sdiv{{\mathsf{div}}}
\newcommand\Sconf{{\mathsf{conf}}}
\newcommand\SIn{{\mathsf{in}}}
\newcommand\SOut{{\mathsf{out}}}
\newcommand\Sin{{\mathsf{in}}}
\newcommand\Sout{{\mathsf{out}}}

\newcommand\Sapp{{\mathsf{app}}}
\newcommand\Sdis{{\mathsf{dis}}}

\newcommand\SSP[2]{\langle #1, #2 \rangle} 

\newcommand\SVar{x}
\newcommand\SCost{\theta}
\newcommand\SRepa{\lambda}
\newcommand\SCostRepa{\SCost^\SRepa}
\newcommand\SMinMar{\Theta}

\begingroup
  \newcommand\CreateFactorMacro[2]{
    \def\name{#1}
    \def\symbol{#2}

    \expandafter\xdef\csname S\name Det\endcsname{\expandonce\symbol_{\noexpand\Sdet}}
    \expandafter\xdef\csname S\name In\endcsname{\expandonce\symbol_{\noexpand\Sin}}
    \expandafter\xdef\csname S\name Out\endcsname{\expandonce\symbol_{\noexpand\Sout}}
    \expandafter\xdef\csname S\name Conf\endcsname{\expandonce\symbol_{\noexpand\Sconf}}

    \expandafter\xdef\csname S\name Factor\endcsname ##1{\expandonce\symbol_{##1}}
    \expandafter\xdef\csname S\name FactorDet\endcsname ##1{\expandonce\symbol_{##1,\noexpand\Sdet}}
    \expandafter\xdef\csname S\name FactorIn\endcsname ##1{\expandonce\symbol_{##1,\noexpand\Sin}}
    \expandafter\xdef\csname S\name FactorOut\endcsname ##1{\expandonce\symbol_{##1,\noexpand\Sout}}

    \expandafter\xdef\csname S\name FactorConf\endcsname ##1{\expandonce\symbol_{##1}}
  }

  \CreateFactorMacro{Var}{\SVar}
  \CreateFactorMacro{Cost}{\SCost}
  \CreateFactorMacro{CostRepa}{\SCostRepa}
  \CreateFactorMacro{MinMar}{\SMinMar}
\endgroup

\newcommand\SEdgeTrans[2]{{#1\!\rightarrow\! #2}}
\newcommand\SHatEdgeTrans[3]{{#1\!\stackrel{#3}{\rightarrow}\! #2}}
\newcommand\SEdgeDiv[3]{{#1\!\rightrightarrows\! #2/#3}}
\newcommand\SHatEdgeDiv[4]{{#1\!\stackrel{\text{\raisebox{-2pt}[0pt][0pt]{$#4$}}}{\rightrightarrows}\! #2/#3}}
\newcommand\SEdgeConf[2]{{#1 \lightning #2}}

\newcommand\SMsg{\Delta}
\newcommand\SMsgLeft[1]{\SMsg^\leftarrow_{#1}}
\newcommand\SMsgRight[1]{\SMsg^\rightarrow_{#1}}
\newcommand\SMsgFromConf[1]{\SMsg^\uparrow_{#1}}
\newcommand\SMsgFromFact[1]{\SMsg^\uparrow_{#1}}

%% file: figures/decomposition.tex
\begin{figure*}
  \centering
  \includegraphics[width=\textwidth]{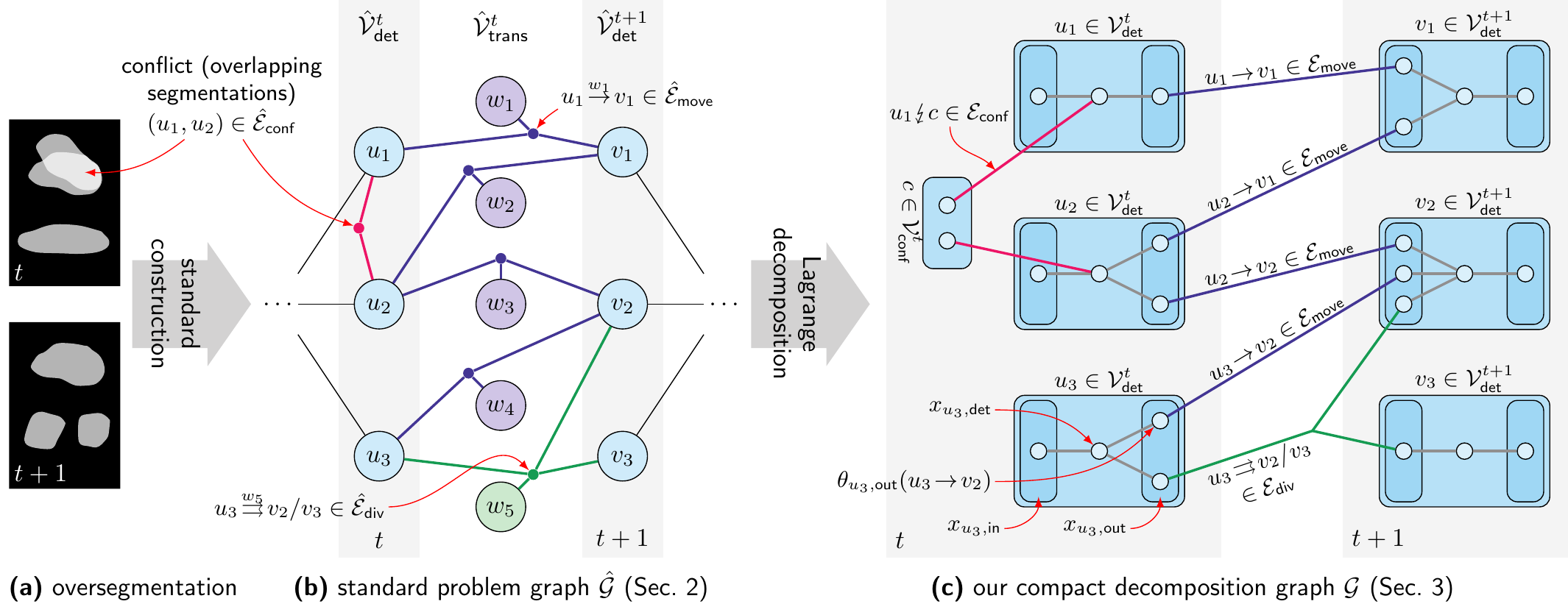}%
  \vspace{-4pt}%
  \newcommand\TmpRef[1]{\textsf{\small\bfseries (#1)}}%
  \caption{%
    Illustration of the standard model and our proposed problem decomposition.
    Each detection (segmentation hypothesis) in~\TmpRef{a} is represented by a node in~$\hat\SV_\Sdet$ in the standard model~\TmpRef{b}.
    Possible transitions between time steps are represented by a node in~$\hat\SV_\Strans$.
    Hyper-edges~$\hat\SE$ handle detection and transition coupling constraints.
    Our Lagrange decomposition~\TmpRef{c} represents each detection together with its incoming and outgoing transitions by a single node in~$\SV_\Sdet$.
    Conflicting detections are represented by separate nodes in~$\SV_\Sconf$.
    Edges~$\SE$ correspond to coupling constraints between nodes and refer to transitions~(blue), divisions~(green), and conflicts~(red).}
  \label{fig:decomposition}
  \vspace{-1em}
\end{figure*}

%% file: algorithm_dual.tex

\def\MyAwesomeAlgoSkip{}

\DontPrintSemicolon
\SetAlgoHangIndent{1em}
\SetAlgoSkip{MyAwesomeAlgoSkip}

\begin{algorithm}[t]
  \gdef\MyAwesomeAlgoSkip{\vspace{-1em}}
  \addtolength{\hsize}{1.5em}%
  $T' \leftarrow \{1, \ldots, T\}$;
  $\lambda \leftarrow 0$\;
  \While{not converged}{
    \For{$t \in T'$}{
      \For{$v \in \SV_\Sdet^t$}{
        compute $\SMsgFromFact v$ on $\theta^\lambda$ ;
        $\lambda \leftarrow \lambda + \SMsgFromFact v$\;
      }

      \For{$c \in \SV_\Sconf^t$}{
        compute $\SMsgFromConf c$ on $\theta^\lambda$;
        $\lambda \leftarrow \lambda + \SMsgFromConf c$\;
      }

      estimate assignment (Alg.~\ref{alg:primal}, optional)\;

      \For{$v \in \SV_\Sdet^t$}{
        compute $\SMsgRight v$ on $\theta^\lambda$ ;
        $\lambda \leftarrow \lambda + \SMsgRight v$
        (backward direction: use $\SMsgLeft v$)
      }
    }
    reverse the order of $T'$\;
  }
  \caption{Dual optimization}
  \label{alg:dual}
\end{algorithm}


%% file: algorithm_primal.tex

\def\MyAwesomeAlgoSkip{}

\DontPrintSemicolon
\SetAlgoHangIndent{1em}
\SetAlgoSkip{MyAwesomeAlgoSkip}

\begin{algorithm}[t]
  \gdef\MyAwesomeAlgoSkip{\vspace{-1em}}
  \addtolength{\hsize}{1.5em}%
  $\mu^\star \leftarrow
    \argmin\limits_{\raisebox{0pt}[5pt]{$\scriptstyle\mu\in\{0,1\}^{|\SV_\Sdet^t|}$}} \SSP{s}{\mu}$
    s.t.
    $\sum\limits_{v\in c} \mu_v \le 1\; \forall\, c \in \SV_\Sconf^t$
  \hfill\eqref{equ:conflict_resolution}\par
  \For{$v \in \SV_\Sdet^t$}{
    $x_v \leftarrow \mathsf{OFF}$\hspace{.5em}\textbf{if} $\mu_v^\star = 0$\;
  }

  order elements $v$ of $\SV_\Sdet^{\smash{t}}$ by their score $s(v)$\;

  \For{$v \in \SV_\Sdet^t$ \textup{with} $x_v \ne \mathsf{OFF}$}{
      \vspace{2pt}%
      $\SX'_v \leftarrow \Bigl\{ x\in\SX_v \;\;\Bigm| \text{\small\begin{tabular}{c} $x$ does not violate \\[-1pt] coupling constraints \end{tabular}} \Bigr\}$ \;
      assign $\SVarFactorIn v$, $\SVarFactorDet v$ according to $\argmin_{x\in\SX'_v} \SSP{\SCostRepaFactor v}{x}$
      \vspace*{-3pt}\newline(backward direction: use $\SVarFactorOut v$ instead)\;
      propagate state $\SVarFactor v$ accross edges in $\SE$\;
  }

  \caption{Primal heuristics for time step $t$}
  \label{alg:primal}
\end{algorithm}


%% file: figures/plots.tex
\begin{figure}
  \hspace{-2pt}%
  \includegraphics[trim=2mm 2mm 2mm 2mm]{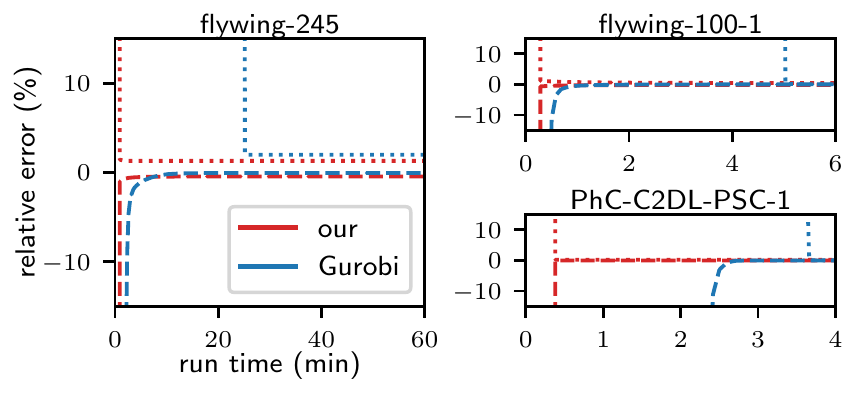}%
  \caption[]{
    Lower- (\includegraphics{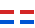}) and upper-bound (\includegraphics{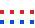}) for our solver and Gurobi on selected instances.
    We obtain high-quality solutions after only a few iterations.}
  \label{fig:plots}
  \vspace{-.8em}
\end{figure}


%% file: tables/results.tex

\newcommand\TmpHead[2]{\multicolumn{#1}{c}{\bfseries #2}}
\newcommand\TmpUnit[1]{\tiny (#1)}

\begin{table*}
  \centering\small
  \begin{tabular}{l rrrr rrrr r@{}lr@{}l}
    \toprule
    \TmpHead{1}{Instance} & \TmpHead{4}{Our Solver}                   & \TmpHead{4}{Gurobi}                                                                                                        & \TmpHead{4}{\clap{Improvement}}\\
                          & time       & mem.          & err.         & TRA & \multicolumn{3}{r}{time relax.\,/\,equ.\,/\,opt.} & mem.          & \multicolumn{2}{c}{\clap{time}} & \multicolumn{2}{c}{\clap{mem.}}\\[-3pt]
                          & \TmpUnit s & \TmpUnit{MiB} & \TmpUnit{\%} &     & \TmpUnit s & \TmpUnit s & \TmpUnit s                                                                 & \TmpUnit{MiB} \\[-1pt]
    \cmidrule(r){1-1}
    \cmidrule(lr){2-5}
    \cmidrule(lr){6-9}
    \cmidrule(l){10-13}

    drosophila
    &  \bfseries  2.7
    &  \bfseries  157
    & 0.00
    & 0.9999
    & 0.9
    &  9.2
    &  9.3
    &  1154
    &  \bfseries  3.4 & \sffamily\,x
    &  \bfseries  7.3 & \sffamily\,x \\[.3em]

    flywing-100-1
    &  \bfseries  82.4
    &  \bfseries  474
    & 0.65
    & 0.9867
    & 270.6
    &  301.8
    &  554.9
    &  5522
    &  \bfseries  3.7 & \sffamily\,x
    &  \bfseries  11.6 & \sffamily\,x \\

    flywing-100-2
    &  \bfseries  78.5
    &  \bfseries  490
    & 0.98
    & 0.9826
    & 243.0
    &  275.2
    &  2311.5
    &  8527
    &  \bfseries  3.5 & \sffamily\,x
    &  \bfseries  17.4 & \sffamily\,x \\

    flywing-245
    &  \bfseries  159.8
    &  \bfseries  1192
    & 1.29\rlap{\footnotemark[1]}
    & ---\rlap{\footnotemark[2]}
    & 1159.0
    &  9540.6
    &  ---\rlap{\footnotemark[3]}
    &  20756
    &  \bfseries  59.7 & \sffamily\,x
    &  \bfseries  17.4 & \sffamily\,x \\[.3em]

    Fluo-C2DL-MSC-1
    &  1.9
    &  \bfseries  53
    & 0.38
    & 0.9922
    & 0.0
    &  \bfseries  0.5
    &  \bfseries  0.6
    &  90
    &  0.3 & \sffamily\,x
    &  \bfseries  1.7 & \sffamily\,x \\

    Fluo-C2DL-MSC-2
    &  2.0
    &  \bfseries  50
    & 0.08
    & 0.9863
    & 0.0
    &  \bfseries  0.1
    &  \bfseries  0.1
    &  61
    &  0.1 & \sffamily\,x
    &  \bfseries  1.2 & \sffamily\,x \\

    Fluo-N2DH-GOWT1-1
    &  \bfseries  0.5
    &  \bfseries  58
    & 0.00
    & 1.0000
    & 0.0
    &  0.6
    &  0.6
    &  153
    &  \bfseries  1.2 & \sffamily\,x
    &  \bfseries  2.6 & \sffamily\,x \\

    Fluo-N2DH-GOWT1-2
    &  \bfseries  0.3
    &  \bfseries  65
    & 0.00
    & 1.0000
    & 0.0
    &  1.1
    &  1.1
    &  196
    &  \bfseries  3.3 & \sffamily\,x
    &  \bfseries  3.0 & \sffamily\,x \\

    PhC-C2DL-PSC-1
    &  \bfseries  22.7
    &  \bfseries  930
    & 0.23
    & 0.9952
    & 40.9
    &  219.7
    &  267.6
    &  13199
    &  \bfseries  9.7 & \sffamily\,x
    &  \bfseries  14.2 & \sffamily\,x \\

    PhC-C2DL-PSC-2
    &  \bfseries  17.9
    &  \bfseries  708
    & 0.14
    & 0.9969
    & 22.2
    &  127.5
    &  156.9
    &  9870
    &  \bfseries  7.1 & \sffamily\,x
    &  \bfseries  13.9 & \sffamily\,x \\

    \bottomrule

    \multicolumn{13}{c}{\smaller
      \footnotemark[1]opt.\ unknown, shows rel.\ primal/dual-diff.\ instead\quad%
      \footnotemark[2]opt.\ unknown, no reference available\quad%
      \footnotemark[3]did not terminate within 8\,h}
  \end{tabular}

  \vspace*{-0.5em}%
  \caption{
    Quantitative comparison of our solver and the ILP solver Gurobi.
    We display run time, maximal memory consumption (mem.), relative error of~\eqref{equ:standard_ILP} compared to optimum (err.) and TRA score.
    For Gurobi time of root relaxation (relax.), finding a comparable solution (equ.) and finding an optimal solution (opt.) is reported.
   }%
  \label{tab:results}%
  \vspace*{-1em}%
\end{table*}

\let\TmpHead=\undefined


%% file: supplement.tex
\raggedbottom\fussy
\clearpage
\onecolumn
\thispagestyle{empty}

\appendix\setcounter{section}{1}
\phantomsection
\addcontentsline{toc}{section}{Supplementary Material}
\aistatstitle{Supplementary Material \\[.6em] A Primal-Dual Solver for Large-Scale Tracking-by-Assignment}

\vfilneg

\def\baselinestretch{1}
\setlength\baselineskip{13pt}

\subsection{Project website}

Our project website at \url{https://vislearn.github.io/libct} contains additional information.
At the time writing there we distribute:
(\emph{i})~The source code of our cell-tracking solver,
(\emph{ii})~information about how to obtain the datasets, and
(\emph{iii})~the model parameters that we have used.

\subsection{Tracking-by-Assignment formulation and cost computation}

A description of the mathematical model of the tracking-by-assignment formulation was already given in Section~\ref{sec:standard}.
Even though the reasoning in the paper has no restrictions on the costs, the cost assignment is a crucial step when using the method in practice.
In the following we describe the cost computation that we have used for preparation of this paper, especially for obtaining the results in our experimental evaluation in Section~\ref{sec:experiments}.

The cost $\SCost_u$ associated with each segmentation variable $u \in \hat\SV_\Sdet$ is based on image and object features of the underlying segmentation hypothesis.
All segmentation hypotheses are assigned negative costs in order to promote selection as part of a tracking solution, \ie, a segmentation hypothesis with higher negative cost is more likely to be picked as part of a solution.
Similar to \cite{jug2014optimal,jug2014tracking,kaiser2018moma}, the cost $\SCost_u$ of any segment hypothesis  is chosen according to its area and convexity according to the following rule
\begin{equation}
  \SCost_u = -\alpha_\Sdet \cdot a(u) + \beta_\Sdet \cdot \bigl(|{a_{\mathcal{C}}(u)-a(u)}|\bigr) + \gamma_\Sdet \cdot \max \bigl(0, \bigl(a(u)-A\bigr) \bigr)^{2} \;,
\end{equation}
where $\alpha_\Sdet$, $\beta_\Sdet$ and $\gamma_\Sdet$ are free coefficients, $a(u)$ is the area of the hypothesis $u$, $a_{\mathcal{C}}(u)$ is the area of convex hull of that hypothesis $u$, and $A$ is a free parameter that denotes the upper limit of the range of reasonable object~(segment)~sizes.

The costs for all transitions between time steps (moves and divisions) are set up to reflect the knowledge of biological experts.
For any $\SHatEdgeTrans uvw \in \hat\SE_\Smove$ the associated cost $\SCost_{\SHatEdgeTrans{u\,}{\,v}{w}}$ is given by a function which takes segment size and displacement (of segment centre of mass) between consecutive time points into account.
The cost for a move variable can be written as
\begin{equation}
  \SCost_{\SHatEdgeTrans{u\,}{\,v}{w}} = \alpha_\Smove \cdot \Delta a(u, v) + \beta_\Smove \cdot \Delta p(u, v) \;,
\end{equation}
where $\alpha_\Smove$ and $\beta_\Smove$ are free coefficients, $\Delta a$ and $\Delta p$ represent the change in area and in squared position between two consecutive time points, respectively.

Let $\SHatEdgeDiv{u}{v}{v'}{w} \in \hat\SE_\Sdiv$.
The cost $\SCost_{\SHatEdgeDiv{u\,}{\,v}{v'}{w}}$ for division variable additionally accounts for the fact that a dividing cell typically splits into two equally sized daughter cells, and that the cumulative volume of the daughter cells roughly equals the volume of the mother cell.
The division variable cost is given by
\begin{equation}
  \begin{aligned}
    \SCost_{\SHatEdgeDiv{u\,}{\,v}{v'}{w}} =
    \alpha_\Sdiv &+
    \beta_\Sdiv \cdot \Delta a_{mds}(u,v,v') +
    \gamma_\Sdiv \cdot \Delta a(v,v') +
    \kappa_\Sdiv \cdot \Delta a(v,v')^2 + {}\\
    &+ 0.5 \cdot \rho_\Sdiv \cdot (\Delta p(u,v)^2 + \Delta p(u,v')^2) +
    \sigma_\Sdiv \cdot \Delta p(v,v')^2 +
    \tau_\Sdiv \cdot \Delta r(u, v, v')\;,
  \end{aligned}
\end{equation}
where $\alpha_\Sdiv$, $\beta_\Sdiv$, $\gamma_\Sdiv$, $\kappa_\Sdiv$, $\rho_\Sdiv$, $\sigma_\Sdiv$ and $\tau_\Sdiv$ are free coefficients, $ \Delta a_{mds}(u, v, v') := |a(u) - a(v) - a(v')|$ is the change of area between mother and daughter cells, and $\Delta r(u, v, v')$ is the difference in angular orientation between mother cell and daughter cells.
Overall, the transition costs discourage the deviation from the above mentioned biological rules for any decision variable.
Transition costs are positive and in order to collect the reward (negative costs) for a segmentation hypothesis, a solution needs to pay the price for explaining the past and future of this segment.

Additionally, it is possible for a cell to appear/disappear along the image border (cells moving in/out of the field of view) but costs of appearance  and disappearance are set to be higher for cells further away from the image boundary.
For sake of simplicity our description in Section~\ref{sec:standard} does not include decision variables for appearance or disappearance events.
However, our formulation allows to deactivate all incoming (outgoing) transition variables for a segment to model cell appearance (disappearance), see~\eqref{equ:standard_uniqueness_constraint}.
The costs for appearance and disappearance described below can be incorporated by simply shifting the costs of incoming and outgoing transition variables and the affected segmentation variable by a constant factor.
The cost $\SCost_\Sapp(u)$ and $\SCost_\Sdis(u)$ for an appearance or disappearance of segmentation hypothesis $u\in\hat\SV_\Sdet$ \ are given by
\begin{align}
  \theta_\Sapp(u) &= \alpha_\Sapp \cdot a(u) + \beta_\Sapp \cdot \sqrt{d_b(u)} + \gamma_\Sapp \cdot d_b(u)\;, \\
  \theta_\Sdis(u) &= \alpha_\Sdis \cdot a(u) + \beta_\Sdis \cdot \sqrt{d_b(u)} + \gamma_\Sdis \cdot d_b(u)\;,
\end{align}
where $\alpha_\Sapp$, $\alpha_\Sdis$, $\beta_\Sapp$, $\beta_\Sdis$, $\gamma_\Sapp$ and $\gamma_\Sdis$ are free coefficients and $d_b(u)$ represents the distance of the centre of mass of hypothesis $u$ to the closest image boundary.

All free coefficients and the parameter $A$ are set to sensible values by the engineer of the proposed system.
The values we have used for all reported results are available online at our project website.

\subsection{Source code of our cell-tracking solver}

We implemented the suggested solving scheme in a modern C++ library.
The source code of this implementation is publicly available and we plan to incorporate further improvements in the future.
To make the results presented in this paper reproducible, the repository holding the source code also contains a fixed version which we used during the preparation of this paper.

Along with the library we provide Python 3 bindings which allow to feed a text file representation of cell-tracking problems into the native library to run the solver.

For further information about the implementation and the text formats, please refer to the \texttt{README} file that is bundled with the source code.

Source code repository: \url{https://github.com/vislearn/libct}

\clearpage
\subsection{Detailed information about the datasets}

A description of all used datasets can be found in Section~\ref{sec:experiments}.
Instructions how to obtain the datasets can be found on our project website.
There we also distribute the resulting optimization problems for each cell-tracking instance in a text format and provide all model parameters.

\input{figures/datasets}

\input{tables/models}

\clearpage
\subsection{Detailed convergence plots}

\input{figures/all_plots}

\clearpage
\subsection{Proofs of mathematical statements}

\begin{lemma}
  The optimization objective
  $E(\SCost, \SVar) = \sum_{v \in \SV_\Sdet} \SSP{\SCostFactor v}{\SVarFactor v} + \sum_{c \in \SV_\Sconf} \SSP{\SCostFactor c}{\SVarFactor c}$,
  \cf~\eqref{equ:energy_minimization}, is equivalent to
  \begin{multline}
    E(\SCost, \SVar) =
      \sum_{v \in \SV_\Sdet}                               \SCostFactorDet v    \, \SVarFactorDet v    +
      \sum_{\substack{e=\SEdgeTrans uv \\\in \SE_\Smove}} \SCostFactorOut u(e) \, \SVarFactorOut u(e) +
      \sum_{\substack{e=\SEdgeTrans uv \\\in \SE_\Smove}} \SCostFactorIn  v(e) \, \SVarFactorIn  v(e) +
      {} \\ {} +
      \sum_{\substack{e=\SEdgeDiv   uvw\\\in \SE_\Sdiv}}    \SCostFactorOut u(e) \, \SVarFactorOut u(e) +
      \sum_{\substack{e=\SEdgeDiv   uvw\\\in \SE_\Sdiv}}    \SCostFactorIn  v(e) \, \SVarFactorIn  v(e) +
      \sum_{\substack{e=\SEdgeDiv   uvw\\\in \SE_\Sdiv}}    \SCostFactorIn  w(e) \, \SVarFactorIn  w(e) +
      \sum_{\substack{e=\SEdgeConf  vc \\\in \SE_\Sconf}}   \SCostFactor    c(v) \, \SVarFactor    c(v)
      \;.
      \label{equ:lemma_energy_decomposition}
  \end{multline}
  \label{lem:energy_decomposition}
\end{lemma}
\begin{proof}
  First, we apply the definition of $\SX_v$ for all $v\in\SV_\Sdet$ as well as the definition of $\SX_c$ for all $c\in\SV_\Sconf$.
  Next, we write the inner products in an explicit form.
  \begin{align}
    \sum_{v \in \SV_\Sdet} \SSP{\SCostFactor v}{\SVarFactor v}
    &= \sum_{v \in \SV_\Sdet} \Bigl(
      \SSP{\SCostFactorDet v}{\SVarFactorDet v} +
      \SSP{\SCostFactorIn  v}{\SVarFactorIn  v} +
      \SSP{\SCostFactorOut v}{\SVarFactorOut v}
    \Bigr)
    \nonumber\\
    &=
      \sum_{v \in \SV_\Sdet}
        \SCostFactorDet v \, \SVarFactorDet v +
      \sum_{v \in \SV_\Sdet}
        \sum_{e \in \SIn(v)}
          \SCostFactorIn v(e) \, \SVarFactorIn v(e) +
      \sum_{u \in \SV_\Sdet}
        \sum_{e \in \SOut(u)}
          \SCostFactorOut u(e) \, \SVarFactorOut u(e)
    \label{equ:proof_lemma_det}\\
    \sum_{c \in \SV_\Sconf} \SSP{\SCostFactor c}{\SVarFactor c}
    &=
      \sum_{c \in \SV_\Sconf}
        \sum_{\substack{v\in\SV_\Sdet\colon\\ \SEdgeConf vc\in\SE_\Sconf}}
          \SCostFactor c(v) \, \SVarFactor c(v)
    =
      \sum_{\SEdgeConf vc\in\SE_\Sconf}
        \SCostFactor c(v) \, \SVarFactor c(v)
    \label{equ:proof_lemma_conf}
  \end{align}
  We can now use the definition of $\SIn(\cdot)$ and $\SOut(\cdot)$ to expand the corresponding sums in~\eqref{equ:proof_lemma_det}.
  \begin{align}
    \sum_{v \in \SV_\Sdet}
      \sum_{e \in \SIn(v)}
        \SCostFactorIn v(e) \, \SVarFactorIn v(e)
    &=
    \sum_{v \in \SV_\Sdet}\quad
      \sum_{\mathclap{\substack{u \in \SV_\Sdet\colon\\e=\\\SEdgeTrans uv \in \SE_\Smove}}}
        \SCostFactorIn v(e) \, \SVarFactorIn v(e)
    +
    \sum_{v \in \SV_\Sdet}\quad
      \sum_{\mathclap{\substack{u,w \in \SV_\Sdet\colon\\e=\\\SEdgeDiv uvw \in \SE_\Smove}}}
        \SCostFactorIn v(e) \, \SVarFactorIn v(e)
    +
    \sum_{v \in \SV_\Sdet}\quad
      \sum_{\mathclap{\substack{u,w \in \SV_\Sdet\colon\\e=\\\SEdgeDiv uwv \in \SE_\Smove}}}
        \SCostFactorIn v(e) \, \SVarFactorIn v(e)
    \nonumber\\
    &=
      \sum_{\mathclap{\substack{e=\\\SEdgeTrans uv \in \SE_\Smove}}}
        \SCostFactorIn v(e) \, \SVarFactorIn v(e)
    +
      \sum_{\mathclap{\substack{e=\\\SEdgeDiv uvw \in \SE_\Smove}}}
        \SCostFactorIn v(e) \, \SVarFactorIn v(e)
    +
      \sum_{\mathclap{\substack{e=\\\SEdgeDiv uwv \in \SE_\Smove}}}
        \SCostFactorIn v(e) \, \SVarFactorIn v(e)
    \label{equ:proof_lemma_tmp_1}\\
    \sum_{u \in \SV_\Sdet}
      \sum_{e \in \SOut(u)}
        \SCostFactorOut u(e) \, \SVarFactorOut u(e)
    &=
    \sum_{u \in \SV_\Sdet}\quad
      \sum_{\mathclap{\substack{v \in \SV_\Sdet\colon\\e=\\\SEdgeTrans uv \in \SE_\Smove}}}
        \SCostFactorOut u(e) \, \SVarFactorOut u(e)
    +
    \sum_{u \in \SV_\Sdet}\quad
      \sum_{\mathclap{\substack{v,w \in \SV_\Sdet\colon\\e=\\\SEdgeDiv uvw \in \SE_\Smove}}}
        \SCostFactorOut u(e) \, \SVarFactorOut u(e)
    \nonumber\\
    &=
      \sum_{\mathclap{\substack{e=\\\SEdgeTrans uv \in \SE_\Smove}}}
        \SCostFactorOut u(e) \, \SVarFactorOut u(e)
    +
      \sum_{\mathclap{\substack{e=\\\SEdgeDiv uvw \in \SE_\Smove}}}
        \SCostFactorOut u(e) \, \SVarFactorOut u(e)
    \label{equ:proof_lemma_tmp_2}
  \end{align}

  Substituting the terms in~\eqref{equ:energy_minimization} by~\eqref{equ:proof_lemma_det}, \eqref{equ:proof_lemma_conf}, \eqref{equ:proof_lemma_tmp_1} and~\eqref{equ:proof_lemma_tmp_2} results in equation~\eqref{equ:lemma_energy_decomposition}.
\end{proof}

\begin{corollary}
  For any $\SVar\in\SX$ the optimization objective $E(\SCost, \SVar)$ is equivalent to
  \begin{multline}
    E(\SCost, \SVar) =
      \sum_{v \in \SV_\Sdet} \Bigl( \SCostFactorDet v(e) + \sum_{\mathclap{\substack{c\in\SV_\Sconf\colon\\\SEdgeConf vc\in\SE_\Sconf}}} \SCostFactor c(v) \Bigr) \, \SVarFactorDet v(e) +
      \sum_{\substack{e=\SEdgeTrans uv \\\in \SE_\Smove}} \Bigl(\SCostFactorOut u(e) + \SCostFactorIn v(e)\Bigr) \, \SVarFactorOut u(e) +
      {} \\ {} +
      \sum_{\substack{e=\SEdgeDiv   uvw\\\in \SE_\Sdiv}} \Bigl(\SCostFactorOut u(e) + \SCostFactorIn  v(e) + \SCostFactorIn  w(e)\Bigr) \, \SVarFactorOut u(e)
      \label{equ:corollary_coupling_constraints}
  \end{multline}
  \label{cor:coupling_constraints}
\end{corollary}
\begin{proof}
  Due to $\SVar\in\SX$ we know that the coupling constraints~\eqref{equ:coupling_constraints} hold.
  This means that for a given edge $e = \SEdgeTrans uv \in \SE_\Smove$ it holds that $\SVarFactorOut u(e) = \SVarFactorIn v(e)$ and similarly for divisions and conflict edges.
  We can now regroup the expression~\eqref{equ:lemma_energy_decomposition} of Lemma~\ref{lem:energy_decomposition} and sort all terms by elements of vector $\SVar$ to directly obtain~\eqref{equ:corollary_coupling_constraints}.
\end{proof}

\clearpage
\begin{proposition-supplement}
  $\forall x \in \SX,\  \SRepa\in \Lambda\colon E(\SCost, x) = E(\SCost^\SRepa, x).$
\end{proposition-supplement}
\begin{proof}
  Due to $\SVar \in \SX$ we can apply Coralarry~\ref{cor:coupling_constraints} and hence know that $E(\SCost, \SVar)$ is equivalent to~\eqref{equ:corollary_coupling_constraints}.
  Corolarry~\ref{cor:coupling_constraints} also holds for~$E(\SCostRepa, \SVar)$ and we obtain
  \begin{multline}
    E(\SCostRepa, \SVar) =
      \sum_{v \in \SV_\Sdet} \Bigl( \SCostRepaFactorDet v(e) + \sum_{\mathclap{\substack{c\in\SV_\Sconf\colon\\\SEdgeConf vc\in\SE_\Sconf}}} \SCostRepaFactor c(v) \Bigr) \, \SVarFactorDet v(e) +
      \sum_{\substack{e=\SEdgeTrans uv \\\in \SE_\Smove}} \Bigl(\SCostRepaFactorOut u(e) + \SCostRepaFactorIn v(e)\Bigr) \, \SVarFactorOut u(e) +
      {} \\ {} +
      \sum_{\substack{e=\SEdgeDiv   uvw\\\in \SE_\Sdiv}} \Bigl(\SCostRepaFactorOut u(e) + \SCostRepaFactorIn  v(e) + \SCostRepaFactorIn  w(e)\Bigr) \, \SVarFactorOut u(e)\;.
    \label{equ:proof_corollary_tmp}
  \end{multline}

  By definition of the reparametrized costs $\SCostRepa$ we can simplify each of the following terms into
  \begin{align}
    &\forall\, v\in\SV_\Sdet\colon&
    \SCostRepaFactorDet v(e) + \sum_{\mathclap{\substack{c\in\SV_\Sconf\colon\\\SEdgeConf vc\in\SE_\Sconf}}} \SCostRepaFactor c(v)
    &= \SCostFactorDet v(e) - \sum_{\mathclap{\substack{c\in\SV_\Sconf\colon\\\SEdgeConf vc\in\SE_\Sconf}}} \SRepa(\SEdgeConf vc) + \sum_{\mathclap{\substack{c\in\SV_\Sconf\colon\\\SEdgeConf vc\in\SE_\Sconf}}} \Bigl(\SCostRepaFactor c(v) + \SRepa(\SEdgeConf vc)\Bigr)
    \nonumber\\
    &&
    &= \SCostFactorDet v(e) + \sum_{\mathclap{\substack{c\in\SV_\Sconf\colon\\\SEdgeConf vc\in\SE_\Sconf}}} \SCostFactor c(v)
    \label{equ:proof_corallary_tmp_1}\\
    &\forall\, e=\SEdgeTrans uv \in \SE_\Smove\colon&
    \SCostRepaFactorOut u(e) + \SCostRepaFactorIn v(e)
    &= \SCostFactorOut u(e) - \SRepa(e) + \SCostFactorIn v(e) + \SRepa(e)
    \nonumber\\
    &&
    &= \SCostFactorOut u(e) + \SCostFactorIn v(e)
    \label{equ:proof_corallary_tmp_2}\\
    &\forall\, e=\SEdgeDiv uvw \in \SE_\Sdiv\colon&
    \SCostRepaFactorOut u(e) + \SCostRepaFactorIn v(e) + \SCostRepaFactorIn w(e)
    &= \SCostFactorOut u(e) - \SRepa_v(e) - \SRepa_w(e) + \SCostFactorIn v(e) + \SRepa_v(e) + \SCostFactorIn w(e) + \SRepa_w(e)
    \nonumber\\
    &&
    &= \SCostFactorOut u(e) + \SCostFactorIn v(e) + \SCostFactorIn w(e)
    \label{equ:proof_corallary_tmp_3}
  \end{align}

  Note that all tuples/triples of $\SRepa$ have been cancelling out each other.
  We can now insert~\eqref{equ:proof_corallary_tmp_1}, \eqref{equ:proof_corallary_tmp_2} and~\eqref{equ:proof_corallary_tmp_3} into~\eqref{equ:proof_corollary_tmp} and obtain the same expression as the right-hand side of~\eqref{equ:corollary_coupling_constraints}.
  Due to Coralarry~\ref{cor:coupling_constraints} we now that the very same expression is equivalent to $E(\SCost, \SVar)$, hence $E(\SCostRepa, \SVar) = E(\SCost, \SVar)$.
\end{proof}

\clearpage
\begin{proposition-supplement}
  Dualizing all coupling constraints~\eqref{equ:coupling_constraints} in the objective~\eqref{equ:energy_minimization} yields the \emph{Lagrange dual problem} $\max_{\SRepa\in\Lambda} D(\SRepa)$, where
  \begin{equation}
    D(\SRepa) :=
      \sum_{\mathclap{u \in \SV_\Sdet}}  \; \min_{x_u \in \SX_u} \SSP{\SCostRepaFactor u}{\SVarFactor u} +
      \sum_{\mathclap{c \in \SV_\Sconf}} \; \min_{x_c \in \SX_c} \SSP{\SCostRepaFactor c}{\SVarFactor c}
    \,.
	\tag{\ref{equ:dual_objective}}
  \end{equation}
\end{proposition-supplement}%
\begin{proof}
  To recap, the primal optimization problem is defined as the following, \cf~\eqref{equ:coupling_constraints} and~\eqref{equ:energy_minimization}:

  \begin{equation}
    \min_{\mathclap{x \in \{0,1\}^n}}\;\;\;
      \Bigl[
        E(\SCost, \SVar) =
          \sum_{\mathclap{u \in \SV_\Sdet}}  \SSP{\SCost_u}{x_u} +
          \sum_{\mathclap{c \in \SV_\Sconf}} \SSP{\SCost_c}{x_c}
      \Bigr]
    \quad\text{s.t. }
    \left\{
      \begin{array}{ll}
        \SVarFactorDet u = \SVar_c(u)                                      & \forall\,\SEdgeConf uc \in \SE_\Sconf \\
        \SVarFactorOut u(\SEdgeTrans uv) = \SVarFactorIn v(\SEdgeTrans uv) & \forall\,\SEdgeTrans uv \in \SE_\Smove \\
        \SVarFactorOut u(\SEdgeDiv uvw) = \SVarFactorIn v(\SEdgeDiv uvw)   & \forall\,\SEdgeDiv uvw \in \SE_\Sdiv \\
        \SVarFactorOut u(\SEdgeDiv uvw) = \SVarFactorIn w(\SEdgeDiv uvw)   & \forall\,\SEdgeDiv uvw \in \SE_\Sdiv \\
      \end{array}
    \right.
    \label{equ:proof_primal_objective}
  \end{equation}

  We are now dualizing all the constraints of~\eqref{equ:proof_primal_objective} by introducing a Lagrangean multipler for each equality constraint in~\eqref{equ:proof_primal_objective}.
  In total we have $|\SE_\Sconf| + |\SE_\Smove| + 2\,|\SE_\Sdiv|$ constraints, so to assign a Lagrangean multiplier to each constraint we will write $\SRepa \in \Lambda = \SR^{|\SE_\Sconf| + |\SE_\Smove| + 2\,|\SE_\Sdiv|}$, see the definition in the main paper.
  The Lagrange dual function augmented by the Lagrange multipliers now reads
  \begin{multline*}
    D(\SRepa) = \min_{\SVar\in\SX} \Bigl[
      E(\SCost, \SVar) +
        \sum_{\mathclap{\substack{    \SEdgeConf uc \in \SE_\Sconf}}}   \bigl( \SVar_c(u)         - \SVarFactorDet u    \bigr) \, \SRepa(\SEdgeConf uc) +
        \sum_{\mathclap{\substack{e=\\\SEdgeTrans uv \in \SE_\Smove}}} \bigl( \SVarFactorIn v(e) - \SVarFactorOut u(e) \bigr) \, \SRepa(e) + {} \\ {} +
        \sum_{\mathclap{\substack{e=\\\SEdgeDiv uvw \in \SE_\Sdiv}}}    \bigl( \SVarFactorIn v(e) - \SVarFactorOut u(e) \bigr) \, \SRepa_v(e) +
        \sum_{\mathclap{\substack{e=\\\SEdgeDiv uvw \in \SE_\Sdiv}}}    \bigl( \SVarFactorIn w(e) - \SVarFactorOut u(e) \bigr) \, \SRepa_w(e)
      \Bigr]\;,
  \end{multline*}
  \begin{multline}
    D(\SRepa) = \min_{\SVar\in\SX} \Bigl[
      E(\SCost, \SVar) +
        \sum_{\mathclap{\substack{    \SEdgeConf  uc  \in \SE_\Sconf}}}  \SVar_c(u)       \, \SRepa(\SEdgeConf uc) +
        \sum_{\mathclap{\substack{e=\\\SEdgeTrans uv  \in \SE_\Smove}}} \SVarFactorIn  v(e) \, \SRepa(e) +
        \sum_{\mathclap{\substack{e=\\\SEdgeDiv   uvw \in \SE_\Sdiv}}}   \SVarFactorIn  v(e) \, \SRepa_v(e) +
        \sum_{\mathclap{\substack{e=\\\SEdgeDiv   uvw \in \SE_\Sdiv}}}   \SVarFactorIn  w(e) \, \SRepa_w(e) - {} \\ {} -
        \sum_{\mathclap{\substack{e=\\\SEdgeConf  uc  \in \SE_\Sconf}}}  \SVarFactorDet u(e) \, \SRepa(e) -
        \sum_{\mathclap{\substack{e=\\\SEdgeTrans uv  \in \SE_\Smove}}} \SVarFactorOut u(e) \, \SRepa(e) -
        \sum_{\mathclap{\substack{e=\\\SEdgeDiv   uvw \in \SE_\Sdiv}}}   \SVarFactorOut u(e) \, \SRepa_v(e) -
        \sum_{\mathclap{\substack{e=\\\SEdgeDiv   uvw \in \SE_\Sdiv}}}   \SVarFactorOut u(e) \, \SRepa_w(e)
      \Bigr]\;.
    \label{equ:proof_lagrange_function}
  \end{multline}

  We can now apply Lemma~\ref{lem:energy_decomposition} to replace the term $E(\SCost, \SVar)$ by~\eqref{equ:lemma_energy_decomposition} in~\eqref{equ:proof_lagrange_function}.
  After regrouping the terms and sorting them by elements of $\SVar$ we get
  \begin{multline}
    D(\lambda) = \min_{\SVar\in\SX} \Bigl[
      \sum_{\mathclap{v \in \SV_\Sdet}}                               \overbrace{\Bigl(\SCostFactorDet v - \sum_{\mathclap{\substack{c\in\SV_\Sconf\colon\\\SEdgeConf cv\in\SE_\Sconf}}} \SRepa(\SEdgeConf cv)\Bigr)}^{\SCostRepaFactorDet v} \, \SVarFactorDet v +
      \sum_{\mathclap{\substack{e=\SEdgeTrans uv \\\in \SE_\Smove}}} \overbrace{\Bigl(\SCostFactorOut u(e) - \SRepa(e)\Bigr)}^{\SCostRepaFactorOut u(e)} \, \SVarFactorOut u(e) +
      \sum_{\mathclap{\substack{e=\SEdgeTrans uv \\\in \SE_\Smove}}} \overbrace{\Bigl(\SCostFactorIn  v(e) + \SRepa(e)\Bigr)}^{\SCostRepaFactorIn  v(e)} \, \SVarFactorIn  v(e) +
      {} \\ {} +
      \sum_{\mathclap{\substack{e=\SEdgeDiv   uvw\\\in \SE_\Sdiv}}}  \quad \underbrace{\Bigl(\SCostFactorOut u(e) - \SRepa_v(e) - \SRepa_w(e)\Bigr)}_{\SCostRepaFactorOut u(e)} \, \SVarFactorOut u(e) +
      \sum_{\mathclap{\substack{e=\SEdgeDiv   uvw\\\in \SE_\Sdiv}}}  \quad \underbrace{\Bigl(\SCostFactorIn  v(e) + \SRepa_v(e)\Bigr)}_{\SCostRepaFactorIn v(e)} \, \SVarFactorIn  v(e) +
      {} \\ {} +
      \sum_{\mathclap{\substack{e=\SEdgeDiv   uvw\\\in \SE_\Sdiv}}}  \quad \underbrace{\Bigl(\SCostFactorIn  w(e) + \SRepa_w(e)\Bigr)}_{\SCostRepaFactorIn w(e)} \, \SVarFactorIn  w(e) +
      \sum_{\mathclap{\substack{e=\SEdgeConf  vc \\\in \SE_\Sconf}}} \quad \underbrace{\Bigl(\SCostFactor    c(v) + \SRepa(e)\Bigr)}_{\SCostRepaFactor c(e)} \, \SVarFactor    c(v)
    \Bigr]
    \;.
    \label{equ:proof_proposition_dual_tmp}
  \end{multline}

  Due to Lemma~\ref{lem:energy_decomposition} we know that~\eqref{equ:proof_proposition_dual_tmp} is equivalent to
  $D(\SRepa)
  = \min_{\SVar\in\SX} E(\SCostRepa, \SVar)
  = \min_{\SVar\in\SX} \bigl[
      \sum_{u \in \SV_\Sdet}  \SSP{\SCostRepa_u}{x_u} +
      \sum_{c \in \SV_\Sconf} \SSP{\SCostRepa_c}{x_c} \bigr]$
  which is our unconstrained objective function for the Lagrange dual of~\eqref{equ:proof_primal_objective}.

  As we want to maximize the dual function $D(\SRepa)$ with respect to $\SRepa \in \Lambda$ the Lagrange dual problem reads
  \begin{equation}
    \max_{\SRepa\in\Lambda} \min_{x\in\SX} \Bigl[
      \sum_{u \in \SV_\Sdet}  \SSP{\SCostRepa_u}{x_u} +
      \sum_{c \in \SV_\Sconf} \SSP{\SCostRepa_c}{x_c}
    \Bigr]
    =
    \max_{\SRepa\in\Lambda} \Bigl[
      \sum_{u \in \SV_\Sdet}  \min_{x_v\in\SX_v}  \SSP{\SCostRepa_u}{x_u} +
      \sum_{c \in \SV_\Sconf} \min_{x_c\in\SX_c} \SSP{\SCostRepa_c}{x_c}
    \Bigr]
    \;.
  \end{equation}
\end{proof}

\clearpage
\begin{proposition-supplement}
  Dual updates $\SMsg \in \{ \SMsgLeft u, \SMsgRight u, \SMsgFromFact u \mid u \in \SV_\Sdet \} \cup \{ \SMsgFromConf c \mid c \in \SV_\Sconf \}$ monotonically increase the dual function, \ie~$\forall\lambda\in\Lambda\colon D(\lambda) \leq D(\lambda+\SMsg)$.
\end{proposition-supplement}
\begin{proof}
For all possible choices of $\SMsg$ we want to show
\begin{equation*}
  D(\SRepa) \leq D (\SRepa + \SMsg) ,
\end{equation*}
for any fixed $\SRepa\in\Lambda$, which is equivalent to $0 \leq D (\SRepa + \SMsg) - D (\SRepa)$. Without loss of generality we can assume $\SRepa = 0$, since any reparametrization is linear, and, therefore, $\SCost^{\SRepa+\SMsg} = \bigl(\SCost^\SRepa\bigr)^\SMsg$. So we can just redefine $\SCost$ to match $\SCost^\SRepa$. Thus, it suffices to prove
\begin{equation}
  0 \leq D (\SMsg) - D (0) , \label{eq:update-repa-to-prove}
\end{equation}
for all possible choices of $\SMsg$.

\textbf{Case 1:} Let $\SMsg = \SMsgFromConf c$, $c \in \SV_\Sconf$ arbitrary but fixed. Recall that for all $u \in c$, $e = \SEdgeConf uc$:
\begin{align*}
  \SMsgFromConf c(e) &:=
      - \SCost _c(u) + \frac{1}{2} \bigl[ \SSP{\SCost _c}{z_c^\star} + \SSP{\SCost _c}{z_c^{\star\star}} \bigr],
  \text{ with }
  z_c^\star = \argmin_{x \in \SX_c} \SSP{\SCost _c}{x},\;
  z_c^{\star\star} = \argmin_{\smash{x \in \SX_c\setminus\{z_c^\star\}}} \!\SSP{\SCost _c}{x} \, .
\end{align*}

For convenience, let $B_c := \frac{1}{2} \bigl[ \SSP{\SCost _c}{z_c^\star} + \SSP{\SCost _c}{z_c^{\star\star}} \bigr]$. Note that $\SSP{\SCost _c}{z_c^\star} \leq B_c \leq \SSP{\SCost _c}{x}$ for all $x\in \SX_c\setminus\{z_c^\star\}$ by definition of $z_c^\star$. We now rewrite the difference $D (\SMsgFromConf c) - D (0)$:
\begin{align*}
  D(\SMsgFromConf c) - D (0)
  & =
    \sum_{d \in \SV_\Sconf} \min_{x \in \SX_d} \SSP{\SCost_d^{\SMsgFromConf c}}{x}
    + \sum_{v \in \SV_\Sdet} \min_{x \in \SX_v} \SSP{\SCost_v^{\SMsgFromConf c}}{x}
    - \biggl[
      \sum_{d \in \SV_\Sconf} \min_{x \in \SX_d} \SSP{\SCost_d}{x}
      + \sum_{v \in \SV_\Sdet} \min_{x \in \SX_v} \SSP{\SCost_v}{x}
    \biggr]
    \\
  & =
    \min_{x \in \SX_c} \SSP{\SCost_c^{\SMsgFromConf c}}{x}
    - \min_{x \in \SX_c} \SSP{\SCost_c}{x}
    +
    \sum_{u \in c} \min_{x \in \SX_u} \SSP{\SCost_u^{\SMsgFromConf c}}{x}
    - \sum_{u \in c} \min_{x \in \SX_u} \SSP{\SCost_u}{x}
    \\
  & =
    \min_{x \in \SX_c} \sum_{u\in c} \bigl[ \SCost_c(u) - \SCost_c(u) + B_c \bigr] \cdot x(u)
    - \SSP{\SCost _c}{z_c^\star}
    +
    \sum_{u \in c} \biggl[
        \min_{x \in \SX_u} \SSP{\SCost_u^{\SMsgFromConf c}}{x}
        - \min_{x \in \SX_u} \SSP{\SCost_u}{x}
      \biggr]
    \\
  & =
    \min \bigl\{ 0, B_c \bigr\} - \SSP{\SCost _c}{z_c^\star}
    +
    \sum_{u \in c} \biggl[
        \min_{x \in \SX_u} \bigl( \SSP{\SCost_u}{x} + [ \SCost_c(u) - B_c ]\cdot x_\Sdet \bigr)
        - \min_{x \in \SX_u} \SSP{\SCost_u}{x}
      \biggr]
\\
\intertext{If $z_c^\star(u) = 0$ for all $u \in c$, Equation \eqref{eq:update-repa-to-prove} holds, as in this case $\SCost_c(u) \geq B_c \geq 0$ for all $u\in c$. So we are left with the case that there exists $u^\star\in c$ such that $z_c^\star(u^\star) = 1$. Note that $u^\star$ is unique since $z_c^\star \in \SX_c$, \cf \eqref{equ:conflict_factor}. In particular, $z_c^\star(u) = 0$ for all $u\in c$, $u\neq u^\star$. Furthermore, it is $\SSP{\SCost_c}{z_c^\star} = \SCost_c(u^\star) \leq B_c \leq 0$.
We now obtain:}
  D(\SMsgFromConf c) - D (0)
  & =
    \min \bigl\{ 0, B_c \bigr\} - \SSP{\SCost _c}{z_c^\star}
    +
    \sum_{\mathclap{\substack{u \in c\colon \\ z_c^\star(u) = 0}}} \; \biggl[
        \min_{x \in \SX_u} \bigl( \SSP{\SCost_u}{x} + [ \SCost_c(u) - B_c ]\cdot x_\Sdet \bigr)
        - \min_{x \in \SX_u} \SSP{\SCost_u}{x}
      \biggr]
    \\ & \qquad\qquad
    +
    \min_{x \in \SX_{u^\star}} \bigl( \SSP{\SCost_{u^\star}}{x} + [ \SCost_c(u^\star) - B_c ]\cdot x_\Sdet \bigr)
    - \min_{x \in \SX_{u^\star}} \SSP{\SCost_{u^\star}}{x}
    \\
  & \geq
    B_c - \SSP{\SCost _c}{z_c^\star} +
    \min_{x \in \SX_{u^\star}} \SSP{\SCost_{u^\star}}{x}
    + \SSP{\SCost _c}{z_c^\star} - B_c
    - \min_{x \in \SX_{u^\star}} \SSP{\SCost_{u^\star}}{x} = 0
\end{align*}
Hence, $D(\SMsgFromConf c) - D (0) \geq 0$.

\textbf{Case 2:} Let $\SMsg = \SMsgFromFact u$, $u\in \SV_\Sdet$ arbitrary but fixed. Recall that for all $e\in \Sconf(u)$:
\begin{equation*}
  \SMsgFromFact u(e) :=
  \min\limits_{x\in\SX_u \colon x_\Sdet = 1} \frac{\SSP{\SCost _u}{x}}{|\Sconf (u)|}
  = \frac{1}{|\Sconf (u)|}\, \min\limits_{x\in\SX_u \colon x_\Sdet = 1} \SSP{\SCost _u}{x} \,.
\end{equation*}
Now, rewriting the difference $D(\SMsgFromFact u) - D (0)$ yields:
\begin{flalign*}
  & D(\SMsgFromFact u) - D (0) && \\
  & =
    \sum_{v \in \SV_\Sdet} \min_{x \in \SX_v} \SSP{\SCost_v^{\SMsgFromFact u}}{x}
    + \sum_{c \in \SV_\Sconf} \min_{x \in \SX_c} \SSP{\SCost_c^{\SMsgFromFact u}}{x}
    -
    \biggl[
      \sum_{v \in \SV_\Sdet} \min_{x \in \SX_v} \SSP{\SCost_v}{x}
      + \sum_{c \in \SV_\Sconf} \min_{x \in \SX_c} \SSP{\SCost_c}{x}
    \biggr]
    \\
  & =
    \min_{x \in \SX_u} \SSP{\SCost_u^{\SMsgFromFact u}}{x}
    - \min_{x \in \SX_u} \SSP{\SCost_u}{x}
    +
    \sum_{\substack{c \in \SV_\Sconf \colon \\ u\in c}} \! \min_{x \in \SX_c} \SSP{\SCost_c^{\SMsgFromFact u}}{x}
    - \!\!\! \sum_{\substack{c \in \SV_\Sconf \colon \\ u\in c}} \! \min_{x \in \SX_c} \SSP{\SCost_c}{x}
  \\
  &&&\text{\scriptsize\textit{proof continues on next page}}
\end{flalign*}
\begin{flalign*}
  & D(\SMsgFromFact u) - D (0) && \\
  & =
    \min_{x \in \SX_u} \Bigl[
        \SSP{\SCost_u}{x} - x_\Sdet\cdot\sum_{\mathclap{e\in \Sconf(u)}} \SMsgFromFact u(e)
      \Bigr]
    - \min_{x \in \SX_u} \SSP{\SCost_u}{x}
    + \sum_{\substack{c \in \SV_\Sconf \colon \\ u\in c}} \!\!
      \biggl[
        \min_{x \in \SX_c} \bigl(
          \SSP{\SCost_c}{x} + \SMsgFromFact u(\SEdgeConf uc)\cdot x(u)
        \bigr)
        - \min_{x \in \SX_c} \SSP{\SCost_c}{x}
      \biggr]
    \\
  & =
    \min_{x \in \SX_u} \Bigl[
        \SSP{\SCost_u}{x} - x_\Sdet \cdot \!\! \min _{\substack{y\in\SX_u \colon \\ y_\Sdet = 1}} \SSP{\SCost_u}{y}
      \Bigr]
    - \min_{x \in \SX_u} \SSP{\SCost_u}{x}
    + \sum_{\substack{c \in \SV_\Sconf \colon \\ u\in c}} \!\!
      \biggl[
        \min_{x \in \SX_c} \bigl(
          \SSP{\SCost_c}{x} + \tfrac{x(u)}{|\Sconf(u)|} \cdot \!\! \min _{\substack{y\in\SX_u \colon \\ y_\Sdet = 1}} \SSP{\SCost_u}{y}
        \bigr)
        - \min_{x \in \SX_c} \SSP{\SCost_c}{x}
      \biggr]
  \\
  & =
    \min \Bigl\{
        0, \min_{\substack{x \in \SX_u \colon \\ x_\Sdet = 1}} \SSP{\SCost_u}{x} -\!\! \min _{\substack{y\in\SX_u \colon \\ y_\Sdet = 1}} \SSP{\SCost_u}{y}
      \Bigr\}
    \!-\! \min_{x \in \SX_u} \SSP{\SCost_u}{x}
    + \!\!\! \sum_{\substack{c \in \SV_\Sconf \colon \\ u\in c}} \!\!
      \biggl[
        \min_{x \in \SX_c} \bigl(
          \SSP{\SCost_c}{x} + \tfrac{x(u)}{|\Sconf(u)|} \cdot \!\! \min _{\substack{y\in\SX_u \colon \\ y_\Sdet = 1}} \SSP{\SCost_u}{y}
        \bigr)
        - \min_{x \in \SX_c} \SSP{\SCost_c}{x}
      \biggr]
  \\
  & \geq
    0 - \min_{x \in \SX_u} \SSP{\SCost_u}{x}
    + \sum_{\substack{c \in \SV_\Sconf \colon \\ u\in c}} \biggl[
        \min_{x \in \SX_c} \SSP{\SCost_c}{x}
        + \min_{x \in \SX_c} \Bigl(
          \tfrac{x(u)}{|\Sconf(u)|} \cdot \!\!\min _{\substack{y\in\SX_u \colon \\ y_\Sdet = 1}} \SSP{\SCost_u}{y}
        \Bigr)
        - \min_{x \in \SX_c} \SSP{\SCost_c}{x}
    \biggr]
    \\
  & =
    - \min_{x \in \SX_u} \SSP{\SCost_u}{x}
    + \sum_{\substack{c \in \SV_\Sconf \colon \\ u\in c}} \!\!
      \min_{x \in \SX_c} \Bigl(
        \tfrac{x(u)}{|\Sconf(u)|} \cdot \!\! \min _{\substack{y\in\SX_u \colon \\ y_\Sdet = 1}} \SSP{\SCost_u}{y}
      \Bigr)
    =
    - \min_{x \in \SX_u} \SSP{\SCost_u}{x}
    + \sum_{\substack{c \in \SV_\Sconf \colon \\ u\in c}} \!\!
      \min \Bigl\{
        0, \tfrac{1}{|\Sconf(u)|} \cdot \min _{\substack{y\in\SX_u \colon \\ y_\Sdet = 1}} \SSP{\SCost_u}{y}
      \Bigr\}
    \\
  & =
    - \min_{x \in \SX_u} \SSP{\SCost_u}{x}
    + \min \Bigl\{
        0, \min _{x\in\SX_u, x_\Sdet = 1} \SSP{\SCost_u}{x}
      \Bigr\}
    =
    - \min_{x \in \SX_u} \SSP{\SCost_u}{x}
    + \min _{x\in\SX_u} \SSP{\SCost_u}{x}
    = 0
\end{flalign*}
Hence, $D(\SMsgFromFact u) - D (0) \geq 0$.

\textbf{Case 3:} Let $\SMsg = \SMsgRight u$, $u\in\SV_\Sdet$ arbitrary but fixed. Recall that for all $e\in\SOut(u)$:
\begin{align*}
  \SMsgRight u(e) &:= \!\!
    \min\limits_{\substack{x \in \SX _u\colon \\ x_\SOut (e) = 1}} \!\!\!
        \SSP{\SCost _u}{x} - \SMinMarFactorOut u
      ,\, \text{if}\ e \in \SE_\Smove, &
  (\SMsgRight u)_v(e) &:=  \tfrac{1}{2} \Bigl[\!
          \min\limits_{\substack{x \in \SX _u\colon \\ x_\SOut (e) = 1}} \!\!\!
          \SSP{\SCost _u}{x} \!-\! \SMinMarFactorOut u
        \Bigr]
      ,\, \text{if}\ e=\SEdgeDiv uvw
\end{align*}
where
  $\SMinMarFactorOut u := \min \Bigl\{ 0, \tfrac{1}{2}\bigl[ \SSP{\SCost _u}{x_u^\star} + \SSP{\SCost _u}{(1, x^\star_{u,\SIn}, y^\star_{u,\SOut})} \bigr] \Bigr\}$,
  $x_u^\star := \;\argmin\limits_{\mathclap{x \in \SX_u\colon x_\Sdet = 1}}\; \SSP{\SCost_u}{x}$, and
  $y_u^\star := \;\argmin\limits_{\mathclap{\substack{x \in \SX_u\colon x_\Sdet = 1, \\ x_\SIn \neq x^\star_{u,\SIn},\; x_\SOut \neq x^\star_{u,\SOut}}}}\; \SSP{\SCost_u}{x}$.

Using similar techniques as above we can rewrite the difference $D(\SMsgRight u) - D (0)$ as follows:
\begin{flalign*}
  & D(\SMsgRight u) - D (0) && \\
  & =
    \sum_{v \in \SV_\Sdet} \min_{x \in \SX_v} \SSP{\SCost_v^{\SMsgRight u}}{x}
    + \sum_{c \in \SV_\Sconf} \min_{x \in \SX_c} \SSP{\SCost_c^{\SMsgRight u}}{x}
    - \biggl[
        \sum_{v \in \SV_\Sdet} \min_{x \in \SX_v} \SSP{\SCost_v}{x}
        + \sum_{c \in \SV_\Sconf} \min_{x \in \SX_c} \SSP{\SCost_c}{x}
      \biggr]
    \\
  & =
    \min_{x \in \SX_u} \SSP{\SCost_u^{\SMsgRight u}}{x}
    + \sum_{\substack{\SEdgeTrans {u\,}{\,v} \in \\ \SE_\Smove\,\cap\,\SOut(u)}} \min_{x \in \SX_v} \SSP{\SCost_v^{\SMsgRight u}}{x}
    + \sum_{\substack{\SEdgeDiv {u\,}{\,v}{w} \in \\ \SE_\Sdiv\,\cap\,\SOut(u)}} \Bigl[
        \min_{x \in \SX_v} \SSP{\SCost_v^{\SMsgRight u}}{x}
        + \min_{x \in \SX_w} \SSP{\SCost_w^{\SMsgRight u}}{x}
      \Bigr]
    \\[-.4em]
    & \quad
    - \Biggl(
        \min_{x \in \SX_u} \SSP{\SCost_u}{x}
        + \sum_{\substack{\SEdgeTrans {u\,}{\,v} \in \\ \SE_\Smove\,\cap\,\SOut(u)}} \min_{x \in \SX_v} \SSP{\SCost_v}{x}
        + \sum_{\substack{\SEdgeDiv {u\,}{\,v}{w} \in \\ \SE_\Sdiv\,\cap\,\SOut(u)}} \Bigl[
            \min_{x \in \SX_v} \SSP{\SCost_v}{x}
            + \min_{x \in \SX_w} \SSP{\SCost_w}{x}
          \Bigr]
      \Biggr)
    \\
  & =
    \min_{x \in \SX_u} \biggl[
        \SSP{\SCost_u}{x} \!
        - \hspace*{-.1em} \sum_{\mathclap{e\in\SOut(u)\cap\SE_\Smove}} \SMsgRight u(e)\cdot x_\SOut(e) \!
        - \hspace*{-1.5em} \sum_{\substack{e = \SEdgeDiv {u\,}{\,v}{w} \in \\ \SOut(u)\,\cap\,\SE_\Sdiv}} \hspace*{-1.2em} \bigl[ (\SMsgRight u)_v(e) + (\SMsgRight u)_w(e) \bigr] \! \cdot x_\SOut(e)
      \biggr]
    + \hspace*{-1.5em}
    \sum_{\substack{\vphantom{f}e = \SEdgeTrans {u\,}{\,v} \in \\ \SE_\Smove\,\cap\,\SOut(u)}} \hspace*{-1.5em} \min_{x \in \SX_v} \biggl[
        \SSP{\SCost_v}{x} + \SMsgRight u(e) \cdot x_\SIn(e)
      \biggr]
    \\ & \quad
    + \hspace*{-.3em}
    \sum_{\substack{e = \SEdgeDiv {u\,}{\,v}{w} \in \\ \SE_\Sdiv\,\cap\,\SOut(u)}} \Biggl[
        \min_{x \in \SX_v} \biggl(
            \SSP{\SCost_v}{x} + (\SMsgRight u)_v(e) \cdot x_\SIn(e)
          \biggr)
        + \min_{x \in \SX_w} \biggl(
            \SSP{\SCost_w}{x} + (\SMsgRight u)_w(e) \cdot x_\SIn(e)
          \biggr)
      \Biggr]
    \\
    & \quad
    - \Biggl(
        \min_{x \in \SX_u} \SSP{\SCost_u}{x}
        + \sum_{\substack{\SEdgeTrans {u\,}{\,v} \in \\ \SE_\Smove\,\cap\,\SOut(u)}} \min_{x \in \SX_v} \SSP{\SCost_v}{x}
        + \sum_{\substack{\SEdgeDiv {u\,}{\,v}{w} \in \\ \SE_\Sdiv\,\cap\,\SOut(u)}} \Bigl[
            \min_{x \in \SX_v} \SSP{\SCost_v}{x}
            + \min_{x \in \SX_w} \SSP{\SCost_w}{x}
          \Bigr]
      \Biggr)
\end{flalign*}
\begin{flushright}\scriptsize\textit{proof continues on next page}\end{flushright}

For convenience, we set $B_\SOut := \tfrac{1}{2} \bigl[ \SSP{\SCost_u}{x_u^\star} + \SCost_u(1,x_{u,\SIn}^\star, y_{u,\SOut}^\star)]$. Observe $B_\SOut \geq \SSP{\SCost_u}{x_u^\star}$.
With this we get:
\begin{flalign*}
  & D(\SMsgRight u) - D (0) &&\\
  & =
    \min_{x \in \SX_u} \Biggl[
        \SSP{\SCost_u}{x}
        - \;\sum_{\mathclap{e\in\SOut(u)}}\;\; \biggl(\min_{\substack{y\in\SX_u \colon \\ y_\SOut(e) = 1}} \!\! \SSP{\SCost_u}{y} - \min\{0,B_\SOut\}\biggr)\cdot x_\SOut(e)
      \Biggr]
    \\[-.3em] & \qquad
    +
    \sum_{\substack{e = \SEdgeTrans {u\,}{\,v} \in \\ \SE_\Smove\,\cap\,\SOut(u)}} \min_{x \in \SX_v} \biggl[
        \SSP{\SCost_v}{x} + \biggl(\min_{\substack{y\in\SX_u \colon \\ y_\SOut(e) = 1}} \!\! \SSP{\SCost_u}{y} - \min\{0,B_\SOut\}\biggr) \cdot x_\SIn(e)
      \biggr]
    \\[-.3em] & \qquad
    +
    \sum_{\substack{e = \SEdgeDiv {u\,}{\,v}{w} \in \\ \SE_\Sdiv\,\cap\,\SOut(u)}} \Biggl[
        \min_{x \in \SX_v} \biggl(
            \SSP{\SCost_v}{x} + \frac{1}{2}\biggl[\min_{\substack{y\in\SX_u \colon \\ y_\SOut(e) = 1}} \!\! \SSP{\SCost_u}{y} - \min\{0,B_\SOut\}\biggr] \cdot x_\SIn(e)
          \biggr)
    \\[-.9em]
    & \hspace*{10em}
        + \min_{x \in \SX_w}  \biggl(
            \SSP{\SCost_w}{x} + \frac{1}{2}\biggl[\min_{\substack{y\in\SX_u \colon \\ y_\SOut(e) = 1}} \!\! \SSP{\SCost_u}{y} - \min\{0,B_\SOut\}\biggr] \cdot x_\SIn(e)
          \biggr)
      \Biggr]
    \\ & \qquad
    -
    \biggl(
      \min_{x \in \SX_u} \SSP{\SCost_u}{x}
      +
      \sum_{\substack{\SEdgeTrans {u\,}{\,v} \in \\ \SE_\Smove\,\cap\,\SOut(u)}} \min_{x \in \SX_v} \SSP{\SCost_v}{x}
      +
      \sum_{\substack{\SEdgeDiv {u\,}{\,v}{w} \in \\ \SE_\Sdiv\,\cap\,\SOut(u)}} \Bigl[
          \min_{x \in \SX_v} \SSP{\SCost_v}{x}
          + \min_{x \in \SX_w} \SSP{\SCost_w}{x}
        \Bigr]
    \biggr)
  \\
  & \geq
    \min \Biggl\{
        0,
        \min_{e\in\SOut(u)} \biggl[
            \min_{\substack{x \in \SX_u \colon \\ x_\SOut(e) = 1}}
              \SSP{\SCost_u}{x} - \min_{\substack{y\in\SX_u \colon \\ y_\SOut(e) = 1}} \SSP{\SCost_u}{y} + \min\{0,B_\SOut\}
          \biggr]
      \Biggr\}
    \\ & \qquad
    + \hspace*{-1em}
    \sum_{\substack{\SEdgeTrans {u\,}{\,v} \in \\ \SE_\Smove\,\cap\,\SOut(u)}} \hspace*{-1em}
      \min_{x \in \SX_v} \SSP{\SCost_v}{x}
    + \hspace*{-1em}
    \sum_{\substack{\SEdgeDiv {u\,}{\,v}{w} \in \\ \SE_\Sdiv\,\cap\,\SOut(u)}} \hspace*{-0.5em}
      \biggl[
        \min_{x \in \SX_v} \SSP{\SCost_v}{x} + \min_{x \in \SX_w} \SSP{\SCost_w}{x}
      \biggr]
    + \hspace*{-0.8em}
    \sum_{e\in\SOut(u)} \hspace*{-0.7em}
      \min \Bigl\{
        0, \min _{\substack{y\in\SX_u,\\y_\SOut(e) = 1}} \SSP{\SCost_u}{y} - \min\{0,B_\SOut\}
      \Bigr\}
    \\ & \qquad
    -
    \biggl(
      \min_{x \in \SX_u} \SSP{\SCost_u}{x}
      +
      \sum_{\substack{\SEdgeTrans {u\,}{\,v} \in \\ \SE_\Smove\,\cap\,\SOut(u)}} \min_{x \in \SX_v} \SSP{\SCost_v}{x}
      +
      \sum_{\substack{\SEdgeDiv {u\,}{\,v}{w} \in \\ \SE_\Sdiv\,\cap\,\SOut(u)}} \Bigl[
          \min_{x \in \SX_v} \SSP{\SCost_v}{x}
          + \min_{x \in \SX_w} \SSP{\SCost_w}{x}
        \Bigr]
    \biggr)
  \\
  & =
    \min\{0,B_\SOut\}
    +
    \sum_{e\in\SOut(u)} \min \Bigl\{
        0, \min _{\substack{y\in\SX_u,\\y_\SOut(e) = 1}} \SSP{\SCost_u}{y} - \min\{0,B_\SOut\}
      \Bigr\}
    - \min_{x \in \SX_u} \SSP{\SCost_u}{x}
  \\
  & =
    \min\{0,B_\SOut\}
    + \min \Bigl\{
        0, \SSP{\SCost_u}{x_u^\star} - \min\{0,B_\SOut\}
      \Bigr\}
    - \min \{ 0, \SSP{\SCost_u}{x_u^\star} \} = 0
\end{flalign*}
Hence, $D(\SMsgRight u) - D (0) \geq 0$.

\textbf{Case 4:} Let $\SMsg = \SMsgLeft u$, $u\in \SV_\Sdet$ arbitrary but fixed. In this case the argument is completely analogous to 3.
\end{proof}

\begin{proposition-supplement}
  The maximization of the dual~\eqref{equ:dual_objective} yields the same value as the natural LP relaxation of~\eqref{equ:energy_minimization}, more precisely
  \begin{equation}
    \max_{\SRepa \in \Lambda} \Bigl[ D(\lambda) =
      \sum_{\mathclap{v \in \SV_\Sdet}}  \; \min_{x_v \in \SX_v} \SSP{\SCostRepaFactor v}{\SVarFactor v} +
      \sum_{\mathclap{c \in \SV_\Sconf}} \; \min_{x_c \in \SX_c} \SSP{\SCostRepaFactor c}{\SVarFactor c}
    \Bigr]
    =
    \min_{\substack{x \in [0,1]\\\text{st.~\eqref{equ:coupling_constraints} hold}}} \Bigl[ E(\SCost, x) =
      \sum_{\mathclap{v \in \SV_\Sdet}}  \SSP{\SCostFactor v}{\SVarFactor v} +
      \sum_{\mathclap{c \in \SV_\Sconf}} \SSP{\SCostFactor c}{\SVarFactor c}
    \Bigr]
    \;.
    \label{equ:proposition_lp_relaxation}
  \end{equation}
\end{proposition-supplement}
\begin{proof}
  Instead of showing this result directly we will reference the corresponding general results in the literature as this property is not special to the Lagrange decomposition at hand.
  We refer to the excellent survey by Guignard~\cite{guignard2003lagrangean} that summarizes the Lagrange decomposition technique and gives a number of mathematical and applied insights.
  Generally, it is known that the Lagrange decomposition is always at at least as good as the LP relaxation~\cite{guignard2003lagrangean}, \ie using ``$\le$'' instead of ``$=$'' in~\eqref{equ:proposition_lp_relaxation}.
  If the relaxed solutions for all subproblems of the Lagrange decomposition are integer (\ie the LP relaxation of all subproblems is tight) then the Lagrange decomposition dual is not stronger than the LP relaxation, \ie they have the same optimal value~\cite[Corallary 5.1]{guignard2003lagrangean}.

  In our decomposition we have dualized all coupling constraints~\eqref{equ:coupling_constraints} which leads us to the dual function
  \begin{equation}
    D(\lambda) =
      \sum_{\mathclap{v \in \SV_\Sdet}}  \; \min_{x_v \in \SX_v} \SSP{\SCostRepaFactor v}{\SVarFactor v} +
      \sum_{\mathclap{c \in \SV_\Sconf}} \; \min_{x_c \in \SX_c} \SSP{\SCostRepaFactor c}{\SVarFactor c}
    \;.
	\tag{\ref{equ:dual_objective}}
  \end{equation}
  All subproblems in our dual $D(\SRepa)$ consists of minimizing simple inner products.
  Hence it is trivial to see that the the LP relaxation of all subproblems are tight.
\end{proof}

%% file: figures/datasets.tex
\begin{figure}[h]
  \centering
  \subcaptionbox{drosophila}{\includegraphics[height=4.62cm]{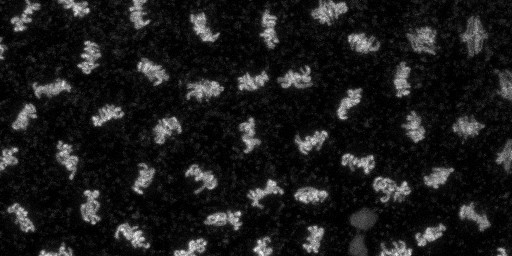}}%
  \hfill
  \subcaptionbox{flywing}{\includegraphics[height=4.62cm]{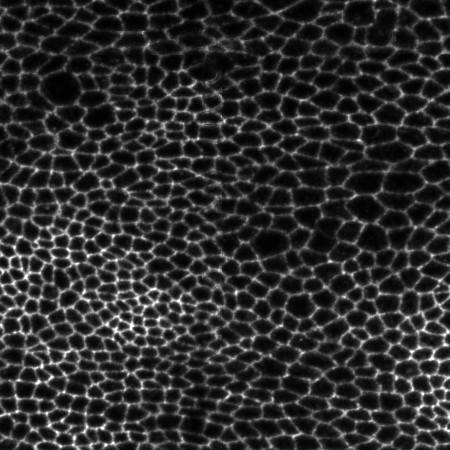}}%

  \vspace{1em}%
  \subcaptionbox{Fluo-C2DL-MSC}{\includegraphics[height=4cm]{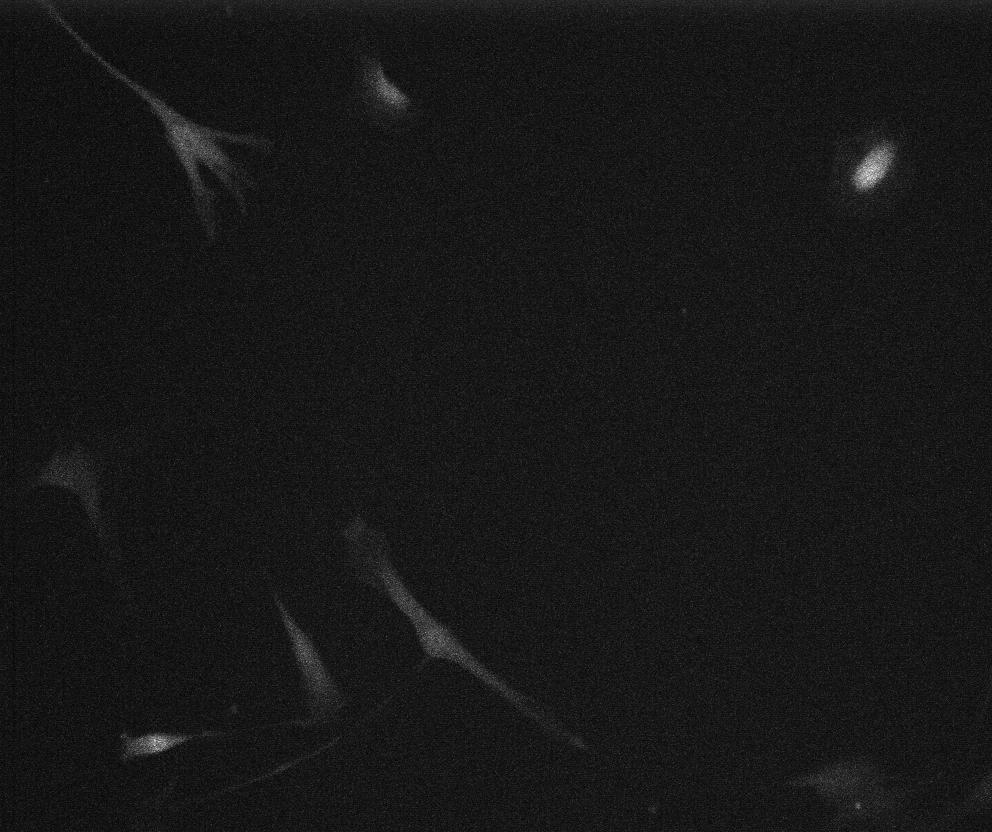}}%
  \hfill
  \subcaptionbox{Fluo-N2DH-GOWT1}{\includegraphics[height=4cm]{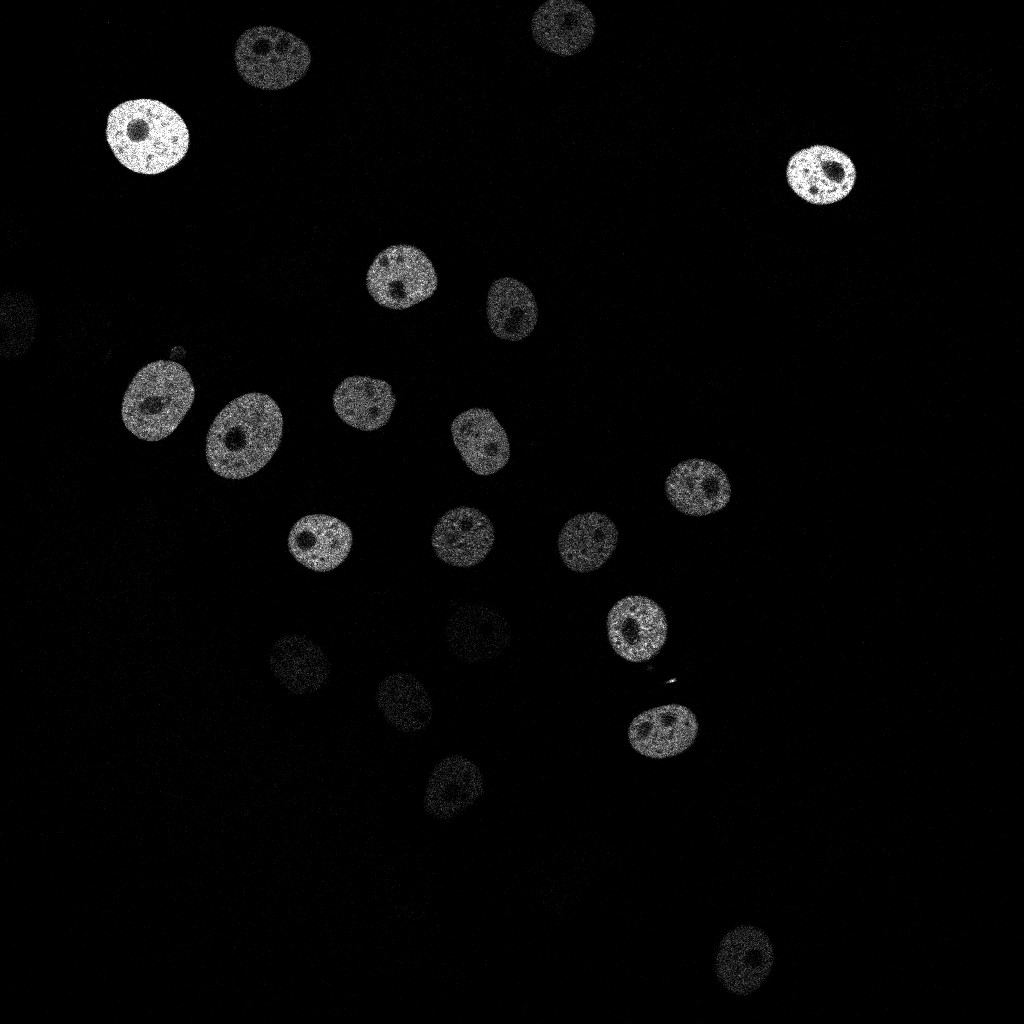}}%
  \hfill
  \subcaptionbox{PhC-C2DL-PSC}{\includegraphics[height=4cm]{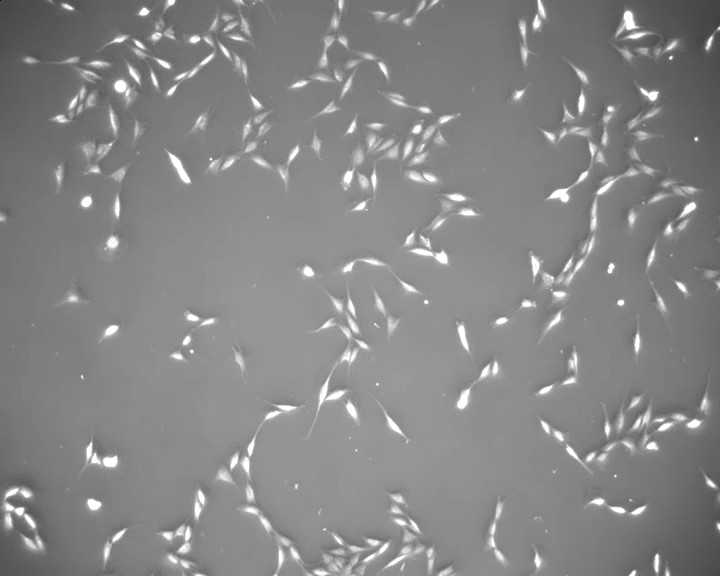}}

  \caption{Example images of datasets that have been used for the evaluation.}
\end{figure}

%% file: tables/models.tex

\begin{table}[h]
  \centering
  \begin{tabular}{l l l l l}
    \toprule
    instance & \#timesteps & \#$\frac{\text{detections}}{\text{time step}}$ & \#$\frac{\text{conflicts}}{\text{time step}}$ & transitive conflict clique\\
    \cmidrule(r){1-1}
    \cmidrule(lr){2-2}
    \cmidrule(lr){3-3}
    \cmidrule(lr){4-4}
    \cmidrule(l){5-5}

    drosophila
    & 252
    & 323.3 $\pm$ 62.9
    & 161.9 $\pm$ 62.9
    & 2.0 $\pm$ 0.3 \\[.5em]

    flywing-100-1
    & 100
    & 2041.1 $\pm$ 358.0
    & 2138.8 $\pm$ 358.0
    & 753.2 $\pm$ 997.2 \\
    
    flywing-100-2
    & 100
    & 2223.4 $\pm$ 258.2
    & 1831.6 $\pm$ 258.2
    & 131.4 $\pm$ 511.7 \\
    
    flywing-245
    & 245
    & 3317.2 $\pm$ 326.6
    & 2733.5 $\pm$ 326.6
    & 54.7 $\pm$ 373.9 \\[.5em]

    Fluo-C2DL-MSC-1
    & 48
    & 115.1 $\pm$ 6.9
    & 41.4 $\pm$ 6.9
    & 12.3 $\pm$ 8.3 \\
    
    Fluo-C2DL-MSC-2
    & 48
    & 52.2 $\pm$ 4.9
    & 18.8 $\pm$ 4.9
    & 9.8 $\pm$ 8.8 \\
    
    Fluo-N2DH-GOWT1-1
    & 92
    & 168.4 $\pm$ 1.6
    & 24.9 $\pm$ 1.6
    & 7.3 $\pm$ 1.6 \\
    
    Fluo-N2DH-GOWT1-2
    & 92
    & 207.1 $\pm$ 4.9
    & 36.8 $\pm$ 4.9
    & 7.5 $\pm$ 2.0 \\
    
    PhC-C2DL-PSC-1
    & 426
    & 1551.4 $\pm$ 482.7
    & 576.7 $\pm$ 482.7
    & 3.4 $\pm$ 1.4 \\
    
    PhC-C2DL-PSC-2
    & 426
    & 1249.8 $\pm$ 372.6
    & 455.8 $\pm$ 372.6
    & 3.5 $\pm$ 1.4 \\
    \bottomrule
  \end{tabular}
  \caption{Characteristics of all used tracking problem instances.}
\end{table}


%% file: figures/all_plots.tex
\begin{figure}[h]
  \newcommand\Tmp[2]{\makebox[\textwidth][l]{\hfill #1\hfill #2\hfill}}%
  \Tmp{\includegraphics{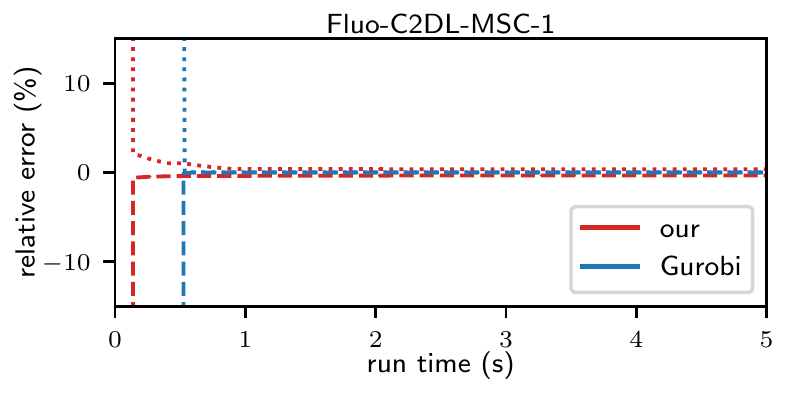}}{\includegraphics{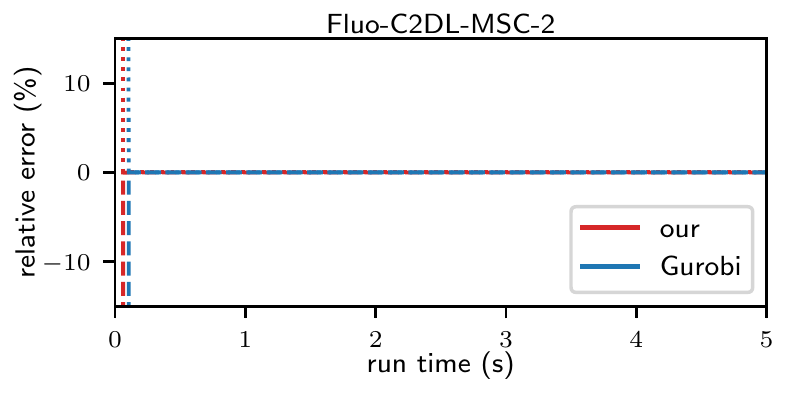}}\par
  \Tmp{\includegraphics{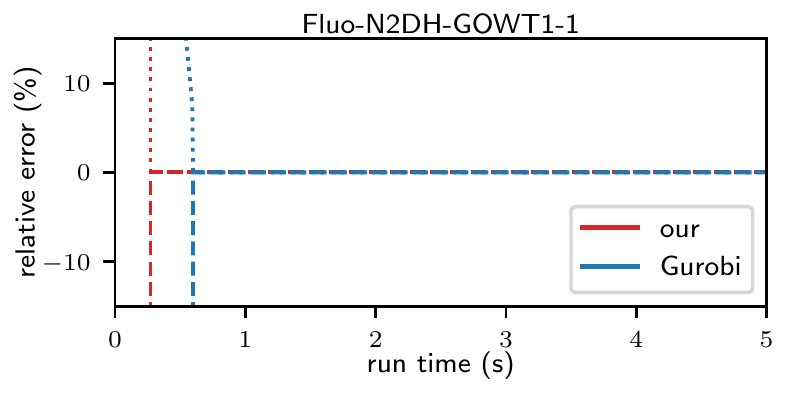}}{\includegraphics{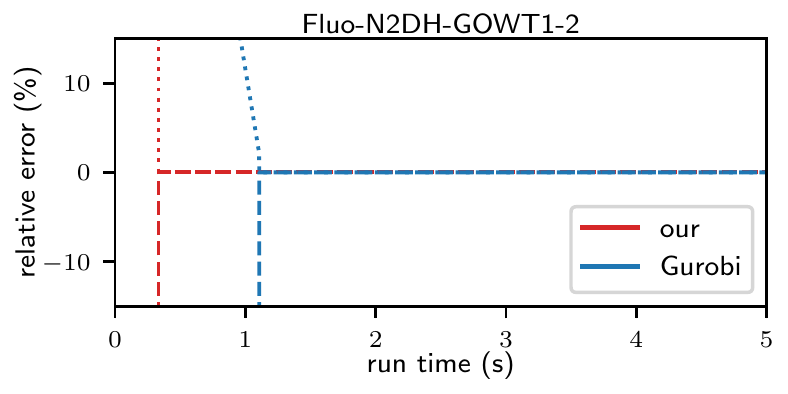}}\par
  \Tmp{\includegraphics{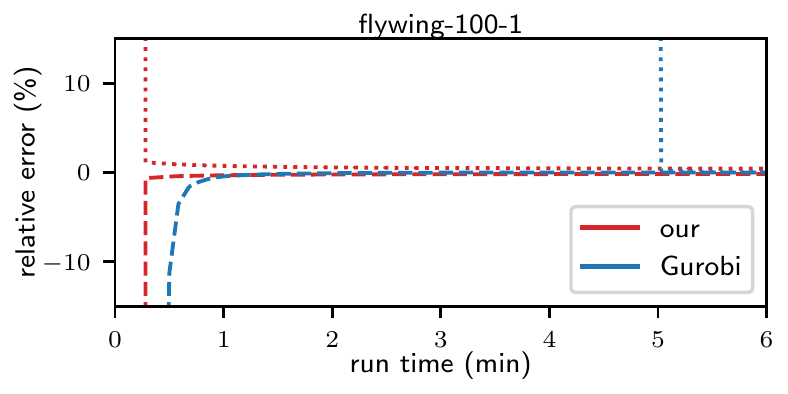}}{\includegraphics{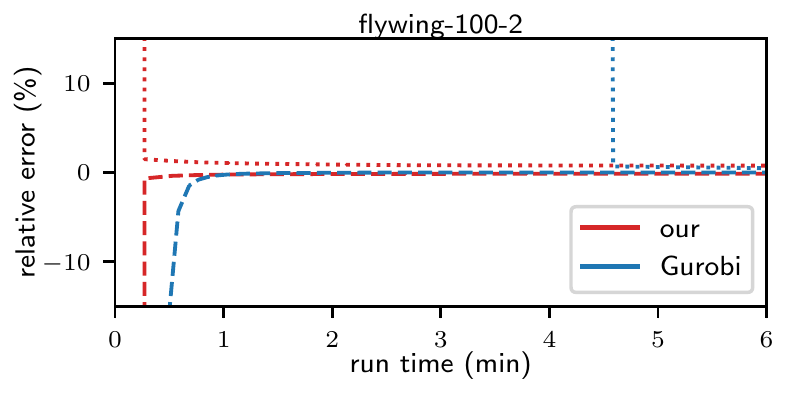}}\par
  \Tmp{\includegraphics{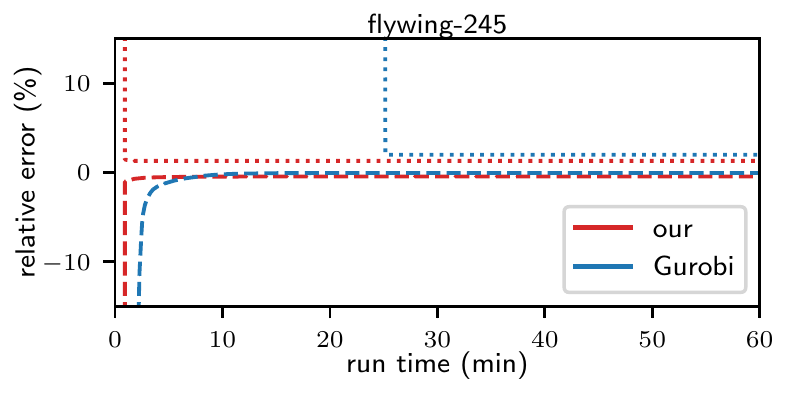}}{\includegraphics{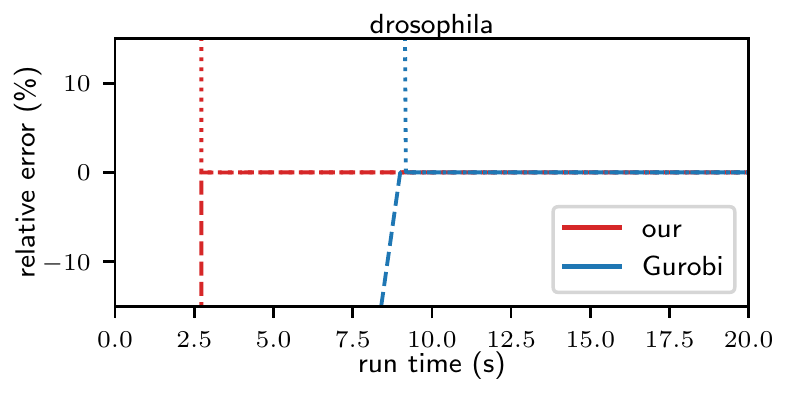}}\par
  \Tmp{\includegraphics{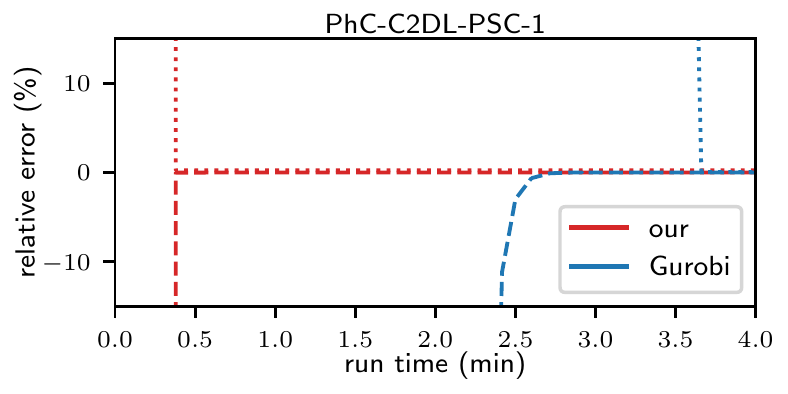}}{\includegraphics{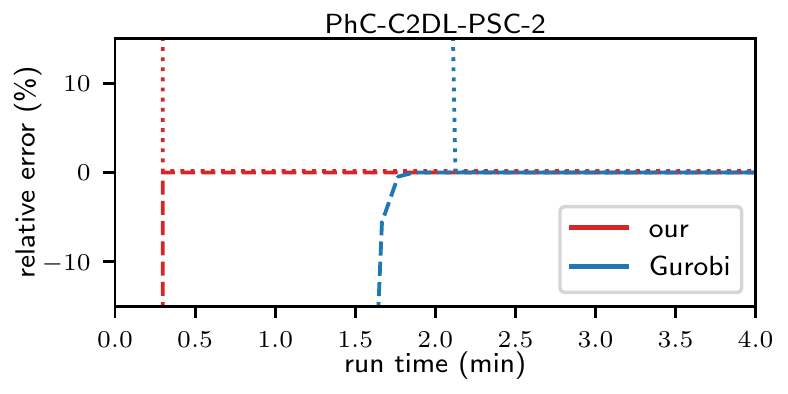}}\par
  \caption[]{
    Comparison of lower-bound (dashed~\includegraphics{images/legend_dashed}) and upper-bound (dotted~\includegraphics{images/legend_dotted}) convergence for our solver and Gurobi.
    We obtain high-quality solutions after only a few iterations.
    For more information see section~\ref{sec:experiments}.}
  \vspace{-3cm}%
\end{figure}
